\documentclass[]{article} 
\usepackage{geometry}
\usepackage{url}
\usepackage[utf8]{inputenc} 
\usepackage[T1]{fontenc}    
\usepackage{lmodern}
\usepackage{booktabs}       
\usepackage{nicefrac}       
\usepackage{microtype}      
\usepackage{setspace}
\usepackage{makecell}
\usepackage{amsmath,amsthm,amsfonts,amssymb}
\usepackage{bm}
\usepackage{appendix}
\usepackage{comment}
\usepackage{tablefootnote}
\usepackage{multirow}
\usepackage{hhline}
\usepackage{graphicx}
\usepackage{subfigure}
\usepackage{algorithm,algorithmic}
\usepackage[font=footnotesize,labelfont=bf]{caption}
\usepackage{url}
\usepackage{authblk}
\usepackage{hyperref}
\usepackage{cases}
\newtheorem{definition}{Definition}
\newtheorem{theorem}{Theorem}

\newtheorem{lemma}{Lemma}

\title{Learning from Similarity-Confidence Data}
\author{
Yuzhou Cao$^{1}$, Lei Feng$^{2}$, Yitian Xu$^{1}$, Bo An$^3$, Gang Niu$^5$ , Masashi Sugiyama$^{4,5}$\\
  $^1$China Agricultural University, College of Science, China\\
  $^2$Chongqing University, College of Computer Science, China\\
  $^3$Nanyang Technological University, School of Computer Science and Engineering, Singapore\\
  $^4$The University of Tokyo, Tokyo, Japan\\
  $^5$RIKEN Center for Advanced Intelligence Project, Tokyo, Japan\\
  \texttt{nanjing.caoyuzhou@gmail.com}, \texttt{lfeng@cqu.edu.cn}\\
  \texttt{xytshuxue@126.com}, \texttt{boan@ntu.edu.sg}\\
  \texttt{gang.niu@riken.jp}, \texttt{sugi@k.u-tokyo.ac.jp}  
}
\date{}
\begin{document}
\maketitle
\begin{abstract}
Weakly supervised learning has drawn considerable attention recently to reduce the expensive time and labor consumption of labeling massive data. In this paper, we investigate a novel weakly supervised learning problem of \textit{learning from similarity-confidence (Sconf) data}, where we aim to learn an effective binary classifier from only unlabeled data pairs equipped with confidence that illustrates their degree of similarity (two examples are similar if they belong to the same class).
To solve this problem, we propose an unbiased estimator of the classification risk that can be calculated from only Sconf data and show that the estimation error bound achieves the optimal convergence rate. To alleviate potential overfitting when flexible models are used, we further employ a risk correction scheme on the proposed risk estimator. Experimental results demonstrate the effectiveness of the proposed methods.
\end{abstract}

\section{Introduction}
\label{S1}
In supervised classification, a vast quantity of exactly labeled data are required for training effective classifiers. However, the collection of massive data with exact supervision is laborious and expensive in many real-world problems. To overcome this bottleneck, \textit{weakly supervised learning} \cite{Zhou2018A} has been proposed and explored under various settings, including but not limited to, \textit{semi-supervised learning} \cite{Chapelle, Xiaojin_Zhu, Gang_Niu, S4VM, PNU, Yufeng_Li, Lanzhe_Guo}, \textit{positive-unlabeled learning} \cite{Elkan, PUa, PUb, EPU}, \textit{noisy-label learning} \cite{Noise_a, CT, Masking, LoIS}, \textit{partial-label learning} \cite{Partial_a, SS_Partial, Consistent_Partial, Deterministic_Partial}, \textit{complementary-label learning} \cite{Comp_a, BComp, Best_Comp, MCL, Limited_MCL, UGE},  \textit{similarity-unlabeled learning} \cite{SU}, and \textit{similarity-dissimilarity learning} \cite{SD}. 

In this paper, we consider a novel weakly supervised learning setting called \textit{similarity-confidence (Sconf) learning}. Under this setting, we aim to train a binary classifier from only unlabeled data pairs equipped with \textit{similarity confidence} that demonstrates the degree of their pairwise similarity, without any ordinarily labeled data. The Sconf learning setting exists in many practical scenarios. Compared with ordinary class labels, similarity labels are more easily accessible in many applications (e.g., protein function prediction \cite{Protein}) and can alleviate potentially biased decisions \cite{Social}. However, such similarity labels could cause severe privacy leakage: for a data pair equipped with a similarity label, the disclosure of the class label of either of the two examples can simultaneously reveal the class label of another one. When the collected data are sensitive (e.g., political opinions and religious orientations), such leakage will lead to serious consequences. In this scenario, similarity confidence is more favorable in the sense of privacy preserving: given the similarity confidence of a data pair, people are uncertain if they share the same label because the confidence only gives the \emph{probability} that they belong to the same class, and it is unable to exactly figure out the underlying similarity label of the data pair from only similarity confidence.

Another example is crowdsourcing \cite{Crowd}. When the data are annotated by crowdworkers, it is difficult for us to always obtain high-quality crowdsourcing labels \cite{Quality} due to the crowdworkers' lack of domain knowledge. When a data pair is annotated with both pairwise similarity and dissimilarity labels by different crowdworkers, we can generate the similarity confidence by averaging instead of choosing the majority of crowdsourcing labels, which can alleviate noisy supervision. In these scenarios, Sconf learning makes it possible to learn an effective binary classifier from only unlabeled data pairs equipped with similarity confidence instead of hard labels. 
\begin{figure*}[t]
\centering
\includegraphics[width=5.5in]{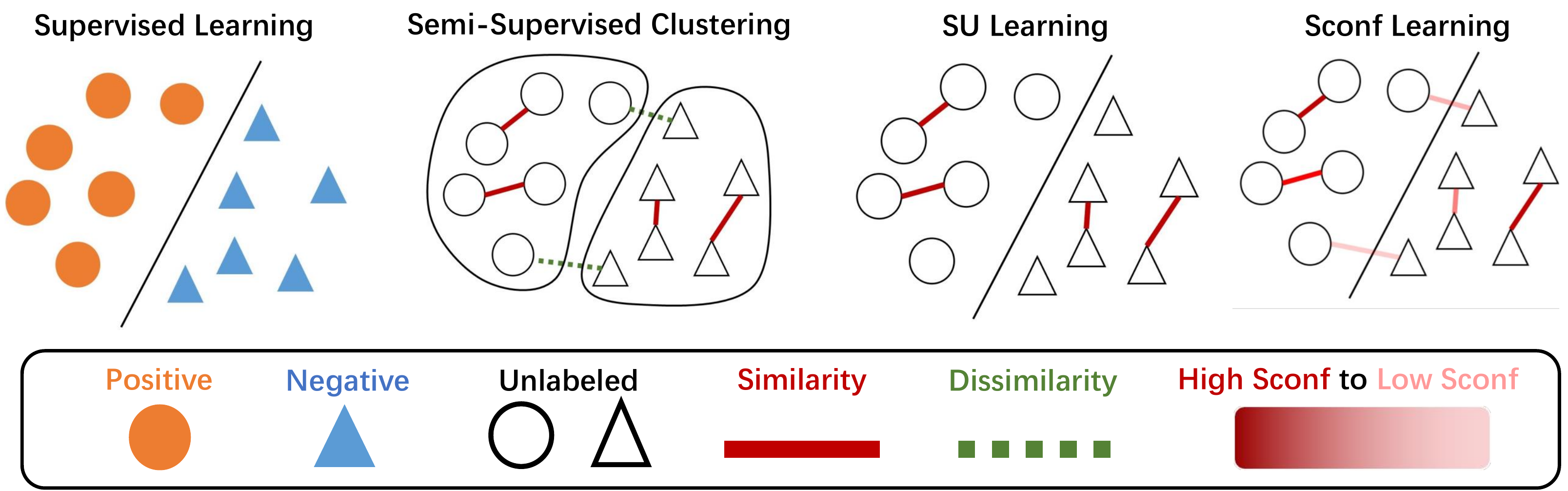}
\caption{Illustration of Sconf learning and related problems.}
\label{F1}
\end{figure*}

Our main contributions in this paper are the following:
\begin{itemize}
\item We propose a novel Sconf learning framework (in Section \ref{S3}) that allows the use of ERM by constructing an unbiased estimator of the classification risk with only unlabeled data pairs with similarity confidence, where any loss functions, models, and optimizers are applicable in this setting. 
\vspace{-5pt}

\item We derive an estimation error bound for Sconf learning and show that it achieves the optimal parametric convergence rate. Analysis on the influence of noisy confidence also shows the robustness of our Sconf learning framework.  
\vspace{-5pt}

\item We leverage an effective empirical risk correction scheme \cite{punn, uunn} for correcting the obtained unbiased risk estimator to alleviate potential overfitting and further show the consistency of its risk minimizer (in Section \ref{S4}).
\vspace{-5pt}

\item Extensive experiments
clearly demonstrate the effectiveness of the proposed Sconf learning method and risk correction scheme  (in Section \ref{S6}). 
\end{itemize}

\section{Related Work}
We illustrate Sconf learning and related problems in Figure \ref{F1}. In what follows, we briefly review \emph{semi-supervised clustering} and \emph{similarity-based learning}.

The research on similarity-based learning was pioneered by semi-supervised clustering (SSC) paradigm, where pairwise similarity/dissimilarity is utilized to enhance the clustering performance \cite{SSC1,SSC2,SSC3,SSC4,SSC5}. From the learning theory viewpoint, the SSC methods are confined in the clustering setting and have no generalization guarantee. 

Recently, many studies have tried to solve the similarity-based learning problem by \emph{empirical risk minimization} (ERM) with rigorous consistency analysis. In \cite{SU}, it was shown that the classification risk can be recovered from similar data pairs and unlabeled data, which enables the use of ERM and analysis on the estimation error. However, the dissimilar data pairs are ignored in this work and the collection of additional unlabeled data is inevitable. 

Later, \cite{SD} made it possible to learn from both similar and dissimilar data pairs by ERM, yet it is still confined within the hard-label setting. \cite{BS} introduced a new performance metric for the binary discriminative model and developed a surrogate risk minimization framework with both similar and dissimilar data pairs. 

On the other hand, the likelihood-based models \cite{Multi-S, Multi-NS} were proposed to conduct similarity-based learning for multi-class classification tasks. The loss functions in these methods are fixed and we cannot directly optimize the classification-oriented losses in \cite{BS}. Compared with these works, our proposed Sconf learning framework is assumption-free on models, loss functions, and optimizers, which makes it a flexible framework when we use deep learning.


\section{Preliminaries}
In this section, we first briefly review the ordinary classification problem and then show our problem setting where each unlabeled data pair is merely equipped with similarity confidence. Proofs are presented in supplementary materials.
\subsection{Ordinary Classification Problem}
Suppose that the feature space is $\textstyle\mathcal{X}\subset\mathbb{R}^{d}$ and the label space is $\mathcal{Y}{\textstyle=\{-1,+1\}}$, the instance and its ordinary class label $(\bm{x}, y)$ obey an unknown distribution with density $p(\bm{x},y)$. Then the critical work is to find a decision function $g(\cdot): \mathcal{X}\rightarrow \mathbb{R}$ that minimizes the classification risk:
\begin{gather}
\label{OR}
R(g)=\textstyle\mathbb{E}_{p(\bm{x},y)}[\ell(g(\bm{x}), y)],
\end{gather}
where $\ell(\cdot, \cdot):~\mathbb{R}\times\mathcal{Y}\rightarrow \mathbb{R}^{+}$ is a binary loss function, e.g., the 0-1 loss and logistic loss. An equivalent expression of classification risk (\ref{OR}) used in the following sections is:
\begin{align}
\nonumber
R(g)&=\textstyle\underbrace{\pi_{+}\mathbb{E}_{+}[\ell(g(\bm{x}),+1)]}_{R_{+}(g)}+\underbrace{\pi_{-}\mathbb{E}_{-}[\ell(g(\bm{x}),-1)]}_{R_{-}(g)}
\end{align}
where $\textstyle\pi_{+}=p(y=+1),~\pi_{-}=p(y=-1)$ denote the class prior probabilities. $\mathbb{E}_{+}[\cdot]$ and $\mathbb{E}_{-}[\cdot]$ are expectations on class-conditional distributions with densities $\textstyle p_{+}(\bm{x})=p(\bm{x}|y=+1)$ and $\textstyle p_{-}(\bm{x})=p(\bm{x}|y=-1)$, respectively. The class posterior probabilities are denoted by $\textstyle r_{+}(\bm{x})=p(y=+1|\bm{x})$ and $\textstyle r_{-}(\bm{x})=p(y=-1|\bm{x})$. 
\subsection{Generation of Similarity-Confidence Data}
\label{s22}
To conduct ERM with only Sconf data, we first give the underlying distributions of Sconf data pairs and further discuss the expression and property of similarity confidence. 

In this setting, we only have access to the unlabeled data pairs with \textbf{similarity confidence}: $\mathcal{S}=\{(\bm{x}_{i}, \bm{x}'_{i}), s_{i}\}_{i=1}^{n}$, where the similarity confidence $s_{i}=s(\bm{x}_{i}, \bm{x}'_{i})=p(y_{i}=y'_{i}|\bm{x}_{i}, \bm{x}_{i}')$ denotes the probability that $\bm{x}_{i}$ and $\bm{x}'_{i}$ share the same label $y_{i}=y_{i}'$. Each unlabeled data pair in $\{(\bm{x}_{i},\bm{x}'_{i})\}_{i=1}^{n}$ is drawn independently from a simple distribution $U^{2}$ with density $p(\bm{x},\bm{x}')=p(\bm{x})p(\bm{x}')$ and we further denote by $\mathcal{S}_{n}$ the unlabeled data pairs with similarity confidence $\{(\bm{x}_{i},\bm{x}'_{i})\}_{i=1}^{n}\mathop{\sim}\limits^{\mathrm{i.i.d.}}p(\bm{x},\bm{x}')$. This formulation implies that we can regard the decoupled unlabeled samples $\{\bm{x}_{i}\}_{i=1}^{n}\cup\{\bm{x}'_{i}\}_{i=1}^{n}$ as drawn from the marginal distribution $p(\bm{x})$ independently, which can be easily implemented in real-world data collection. We also assume the sample independence: $(\bm{x}, y)\perp(\bm{x}', y')$, which is an implicit assumption used in \cite{SU}. Furthermore, we show that the similarity confidence has the following property:
\begin{lemma}(Equivalent expression of similarity confidence)
\label{L1}
\vspace{-5pt}
\begin{gather}
s(\bm{x},\bm{x'})=\textstyle\frac{\pi_{+}^{2}p_{+}(\bm{x})p_{+}(\bm{x}')+\pi_{-}^{2}p_{-}(\bm{x})p_{-}(\bm{x}')}{p(\bm{x})p(\bm{x}')}.
\end{gather}
\end{lemma}

\section{Failure of Learning with One-Sided Similarity Relation}
As mentioned in Section \ref{S1}, though we can recover classification risk from both similar and dissimilar data \cite{SD}, such a type of hard labels could cause serious privacy leakage, which indicates that it is not favorable when the data are sensitive. Such leakage may be alleviated by learning from only one-sided similarity relation: we only have similar (dissimilar) data pairs and no dissimilar (similar) data pairs are provided. For example, when investigating political or religious orientations, people with dissimilar opinions may refuse to give the answer in case of potential conflicts. In these scenarios, only similar data pairs are accessible. As reported in \cite{Contrast}, learning with only similar data pairs can lead to degenerated solutions. A natural optional idea is to combine one-sided similarity relation with similarity confidence. 

Can we learn an effective classifier from only one-sided similarity relation and similarity confidence? Unfortunately, the following experimental and theoretical results give a negative answer to this question. Due to the space limitation, we only provide the result when we only have similar data pairs and a completely analogous result with only dissimilar data pairs is listed in the supplementary materials.
\begin{figure}[!t]
\centering
\subfigure{\includegraphics[width=0.45\textwidth]{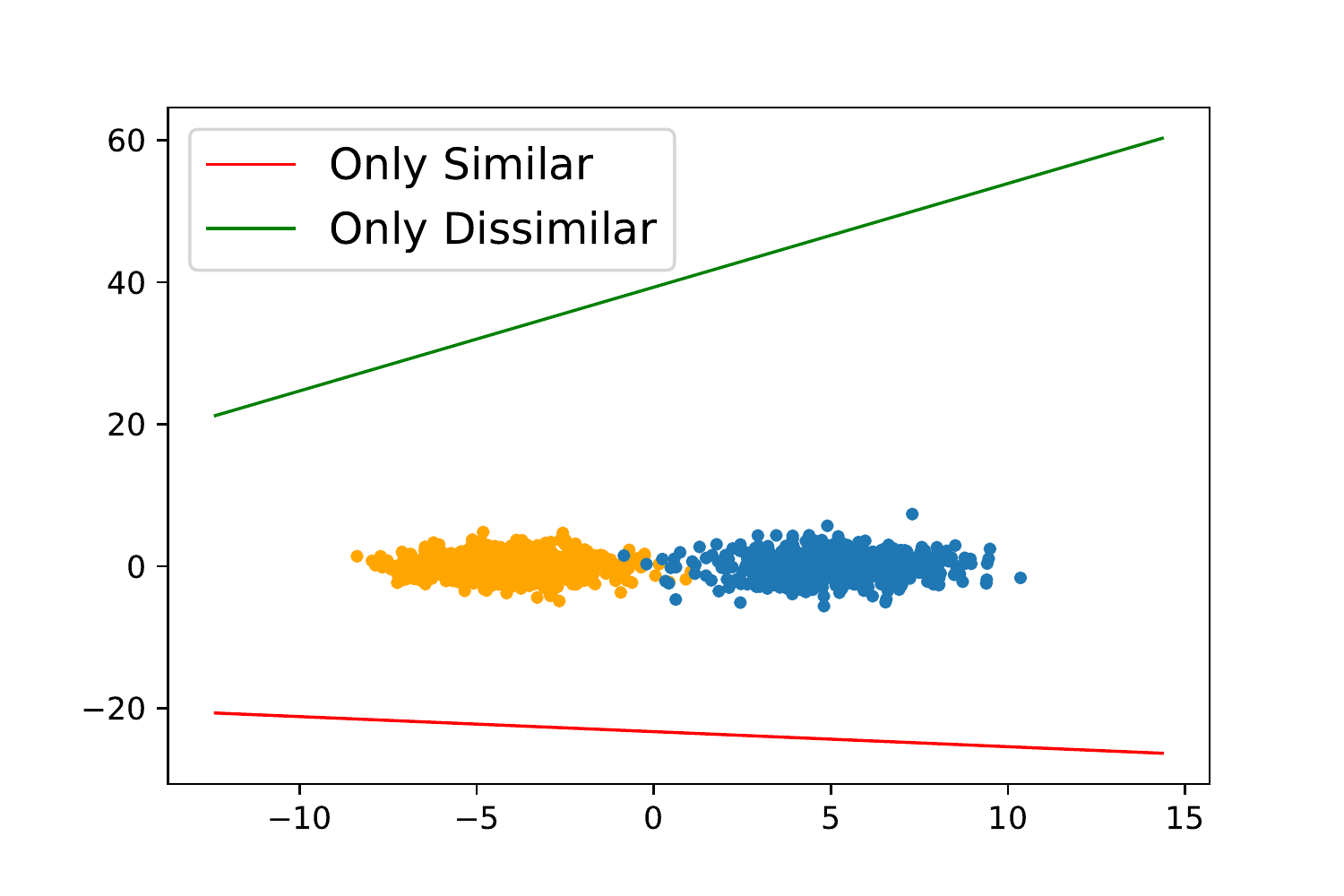}}
\subfigure{\includegraphics[width=0.45\textwidth]{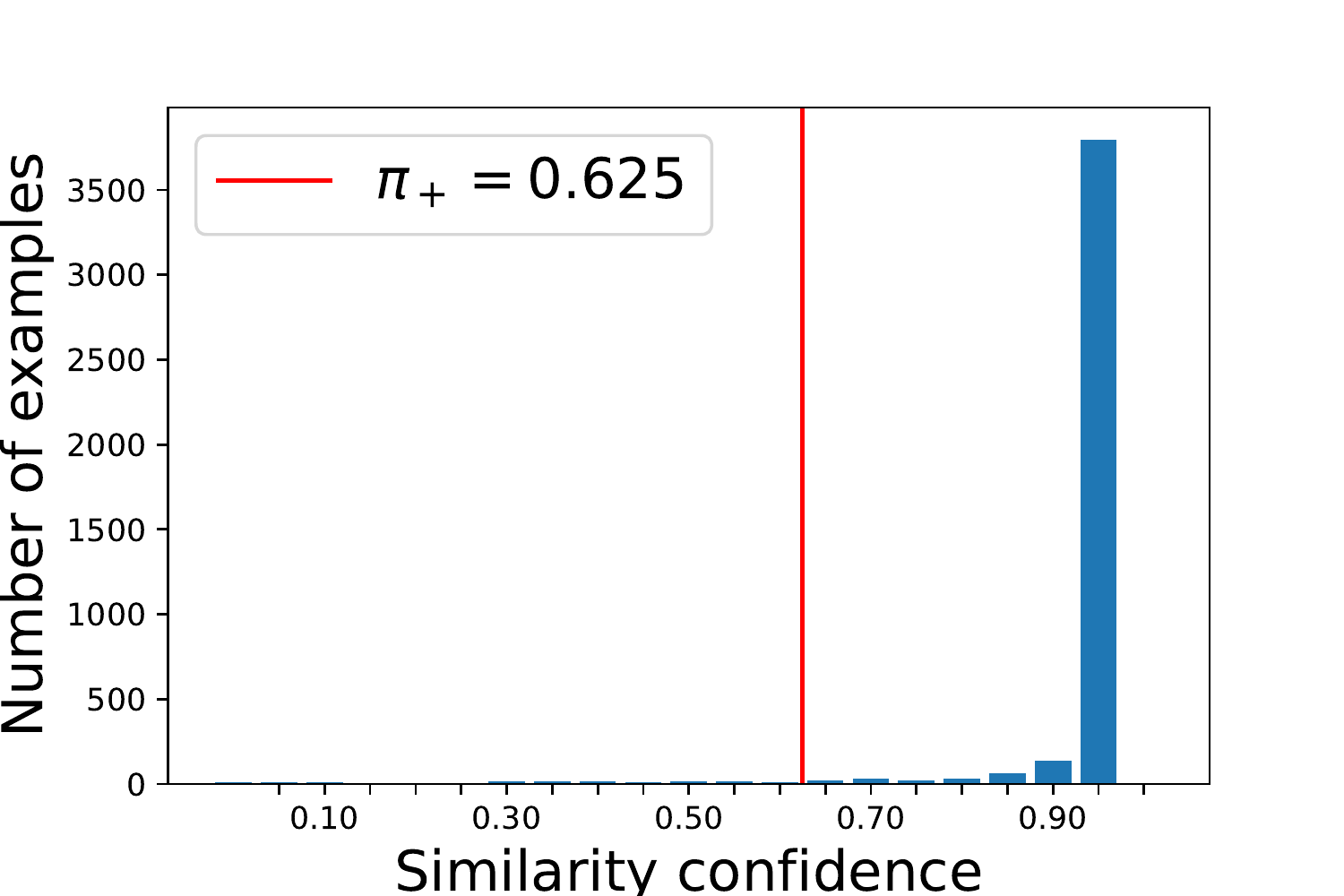}}

\caption{Decision boundaries and the distribution of similarity confidence.}
\label{DD}
\end{figure}

Suppose we have a training set including only similar data pairs:   $$\mathcal{S}^{n}_{\mathrm{S}}=\{(\bm{x}_{i},\bm{x}'_{i})\}_{i=1}^{n}\mathop{\sim}\limits^{\mathrm{i.i.d.}}p(\bm{x},\bm{x}'|y=y'),$$
and their similarity confidences $\{s_{i}\}_{i=1}^{n}$. The following theorem shows that it is \emph{theoretically} possible  to conduct ERM with only data provided above: 
\begin{theorem}
\label{TS}
With similar data pairs and their similarity confidence, assuming that $s(\bm{x},\bm{x}')> 0$ for all the pair $(\bm{x},\bm{x}')$, we can get the unbiased estimator of classification risk (\ref{OR}), i.e.,  ${\textstyle\mathbb{E}_{p(\bm{x},\bm{x}'|y=y')}[\hat{R}_{\mathrm{S}}(g)]=R(g)}$, where
\begin{align}
\label{OS}
&\hat{R}_{\mathrm{S}}(g)=\textstyle(\pi_{+}^{2}+\pi_{-}^{2})\sum\nolimits_{i=1}^{n}\frac{(s_{i}-\pi_{-})(\ell(g(\bm{x}_{i}),+1)+\ell(g(\bm{x}'_{i}),+1))}{2n(\pi_{+}-\pi_{-})s_{i}}
\nonumber\\
&\textstyle +(\pi_{+}^{2}+\pi_{-}^{2})\sum\nolimits_{i=1}^{n}\frac{(\pi_{+}-s_{i})(\ell(g(\bm{x}),-1)+\ell(g(\bm{x}'_{i}),-1))}{2n(\pi_{+}-\pi_{-})s_{i}}.
\end{align}
\end{theorem}
It seems that we can conduct ERM on the obtained unbiased risk estimator to get a binary classifier. Unfortunately, with only one-sided similarity relation, we can only get collapsed solutions \emph{empirically}. Denote the empirical risk minimizer of Eq.~(\ref{OS}) with $\hat{g}_{\mathrm{S}}$ . Then we come to the following conclusion:
\begin{theorem}
\label{coll}
Suppose $\pi_{+}>\pi_{-}$ and  0-1 loss is used. For similar data pairs, we assume that $s_{i}\geq\pi_{+}$ for $i=1,\cdots,n$. Then $\hat{g}_{\mathrm{S}}$ is a collapsed solution that classifies all the examples as positive. 
\end{theorem}
A rough proof intuition for Theorem \ref{coll} is that the coefficients of positive loss terms are always positive and those of negative loss terms are always negative, then minimizing Eq.~(\ref{OS}) is equivalent to minimizing the positive counterpart and maximizing the negative counterpart of classification risk, which can lead to the collapsed solution. We can conclude that though Eq.~(\ref{OS}) is unbiased, it cannot well represent the classification risk in Eq.~(\ref{OR}).

To empirically illustrate the failure of learning with one-sided similarity relation, we conducted experiments on a synthetic dataset and show the distribution of similarity confidence. The detailed statistics of the synthetic dataset is provided in the supplementary materials. According to the experimental results in Figure \ref{DD}, the empirical minimizers of learning with only similar or dissimilar data pairs yield collapsed results and their classification boundaries are severely biased, which aligns with Theorem \ref{coll}. The distribution of similarity confidence also meets our assumption on $\{s_{i}\}_{i=1}^{n}$.  

As shown above, the incorporation of one-sided similarity relation can lead to collapsed solution due to the highly skewed distribution of similarity confidence. A potential remedy for such failure is an underlying non-skewed distribution of similarity confidence. Fortunately, we can achieve this goal with only unlabeled data pairs. In the following section, we show that given unlabeled data pairs with similarity confidence, the hard similarity labels are all completely unnecessary, which means that we can successfully train an effective classifier from only unlabeled data pairs with similarity confidence.

\section{Learning from Similarity-Confidence Data}
\label{S3}
In this section, we propose an unbiased risk estimator for learning from only unlabeled data pairs with similarity confidence and show the consistency of the proposed estimator by giving its estimation error bound. Finally, we propose an effective class-prior estimator for estimating $\pi_{+}$ when it is not given in advance. An analysis of the influence of inaccurate similarity confidence is also provided by giving a high-probability bound.

\subsection{Unbiased Risk Estimator with Sconf Data}
In this section, we derive an unbiased estimator of the classification risk in Eq.~(\ref{OR}) with only Sconf data and establish its risk minimization framework.

Based on the settings in Section \ref{s22}, we first derive the crucial lemma before deriving the unbiased estimator of classification risk (\ref{OR}) from only Sconf data:
\begin{lemma}
\label{LC}
The following equalities hold:

$R_{+}(g)=\mathbb{E}_{U^{2}}[\hat{R}_{+}(g)]$, $R_{-}(g)=\mathbb{E}_{U^{2}}[\hat{R}_{-}(g)]$, where
\begin{align}
\label{cp1}
\textstyle
\hat{R}_{+}(g)&\textstyle=\sum_{i=1}^{n}\frac{(s_{i}-\pi_{-})(\ell(g(\bm{x}_{i}),+1)+\ell(g(\bm{x}'_{i}),+1))}{2n(\pi_{+}-\pi_{-})},\\
\label{cp2}
\textstyle
\hat{R}_{-}(g)&\textstyle=\sum_{i=1}^{n}\frac{(\pi_{+}-s_{i})(\ell(g(\bm{x}_{i}),-1)+\ell(g(\bm{x}'_{i}),-1))}{2n(\pi_{+}-\pi_{-})}.
\end{align}
\end{lemma}
\vspace{-4pt}
According to Lemma \ref{LC} above, we get the unbiased estimator of each counterpart of the classification risk in Eq.~(\ref{OR}). Then we can simply derive the unbiased estimator of the classification risk in Eq.~(\ref{OR}) with only Sconf data:
\begin{theorem}
\label{URE}
We can construct the unbiased estimator $\hat{R}(g)$ of the classification risk (\ref{OR}), i.e., $\mathbb{E}_{U^{2}}[\hat{R}(g)]=R(g)$, with only Sconf data as in the formulation below:
\begin{align}
\textstyle
\label{emp}
\hat{R}(g)&\textstyle=\hat{R}_{+}(g)+\hat{R}_{-}(g).
\end{align}
\end{theorem}
\vspace{-4pt}

Since there are no implicit assumptions on models, losses, and optimizers in our analysis, any convex/non-convex loss and deep/linear model can be used for Sconf learning.
\subsection{Estimation Error Bound}
\label{temporary}
Here we show the consistency of proposed risk estimator $\hat{R}(g)$ in Eq.~(\ref{emp}) by giving an estimation error bound. To begin with, let $\mathcal{G}$ be our function class for ERM. Assume there exists $C_{g}>0$ that $\sup_{g\in\mathcal{G}}\|g\|_{\infty}\leq C_{g}$ and $C_{\ell}>0$ such that $\sup_{|z|\leq C_{g}}\ell(z, y)\leq C_{\ell}$ holds for all $y$. Following the usual practice \cite{foundations}, we assume $\ell(z, y)$ is Lipschitz continuous \textit{w.r.t.} $z$ for all $|z|\leq C_{g}$ and all $y$ with a Lipschitz constant $L_{\ell}$. 

Let $g^{*}=\mbox{arg}\min_{g\in\mathcal{G}} R(g)$ be the minimizer of classification risk in Eq.~(\ref{OR}), and $\hat{g}=\mbox{arg}\min_{g\in\mathcal{G}}\hat{R}(g)$ be the minimizer of empirical risk in Eq.~(\ref{emp}). Then we can derive the following estimation error bound for Sconf learning: 
\begin{theorem}
\label{bound}
For any $\delta>0$, the following inequality holds with probability at least $1-\delta$:
\begin{gather}
\nonumber
\textstyle
R(\hat{g})-R(g^{*})\leq\frac{2L_{\ell}\mathfrak{R}_{n}(\mathcal{G})}{|\pi_{+}-\pi_{-}|}+\frac{2C_{\ell}}{|\pi_{+}-\pi_{-}|}\sqrt{\frac{\ln 2/\delta}{2n}},
\end{gather}
where $\mathfrak{R}_{n}(\mathcal{G})$ is the Rademacher complexity of $\mathcal{G}$ for unlabeled data of size n drawn from the marginal distribution with density $p(\bm{x})$.
\end{theorem}
The definition of the Rademacher complexity \cite{rade} is provided in the supplementary material. Note that the estimation error bound converges in the rate of $\mathcal{O}_{p}(1/\sqrt{n})$ if we assume that $\mathfrak{R}_{n}(\mathcal{G})\leq C_{\mathcal{G}}/\sqrt{n}$, where $\mathcal{O}_{p}$ denotes the order in probability and $C_{\mathcal{G}}$ is a non-negative constant determined by the model complexity. This is a natural assumption since many model classes (e.g., linear-in-parameter models and fully-connected neural network \cite{RadeNN}) satisfy this condition. We make this assumption in the rest of this paper.

Theorem \ref{bound} shows that the empirical risk minimizer converges to the classification risk minimizer with high-probability in the rate of $
\mathcal{O}_{p}(1/\sqrt{n})$. As shown in \cite{Opt}, this is the optimal parametric convergence rate without additional assumptions.
 
\subsection{Class-Prior Estimation from Similarity Confidence}
In our Sconf learning, class-prior $\pi_{+}$ plays an important role in the construction of the unbiased risk estimator. Compared with the previous work \cite{SU, SD}, we make a milder assumption on the data distribution, which aligns with the practical data collection process. However, when the class-prior $\pi_{+}$ is not given, we cannot estimate it by mixture proportion estimation \cite{CPE} as in \cite{SU} since we only have data drawn from a single distribution $U^{2}$. In this section, we propose a simple yet effective class-prior estimator with only Sconf data.

 We have the following theorem for the sample average of similarity confidence: 
\begin{theorem} 
\label{CP}
The sample average of similarity confidence is an unbiased estimator of $\pi_{+}^{2}+\pi_{-}^{2}$:
$${\textstyle\mathbb{E}_{S_{n}}\big[\frac{\sum_{i=1}^{n}s(\bm{x}_{i},\bm{x}'_{i})}{n}\big]=\pi_{+}^{2}+\pi_{-}^{2}.}$$
Furthermore, according to McDiarmid's inequality \cite{mcdiarmid}, the sample average of similarity confidence is consistent and converges to $\pi_{+}^{2}+\pi_{-}^{2}$ in the rate of $\mathcal{O}_{p}(\exp(-n))$, which is the optimal rate according to the central limit theorem \cite{Chung}.
\end{theorem}
Let us denote $\pi_{+}^{2}+\pi_{-}^{2}$ by $\pi_{\mathrm{S}}$. Assuming $\pi_{+}>\pi_{-}$, we can calculate the class prior by $\pi_{+}=\textstyle\frac{\sqrt{2\pi_{\mathrm{S}}-1}+1}{2}$.
According to Theorem \ref{CP}, we can approximate $\pi_{+}$ with the average of similarity confidence and the formulation above.
\subsection{Analysis with Noisy Similarity Confidence}
In the previous sections, we assumed that accurate confidence is accessible. However, this assumption may not be realistic in some practical tasks. We may have the question that how the noisy similarity confidence can affect the learning performance? If our Sconf learning is not robust and even a slight noise on the similarity confidence can cause catastrophic degradation of performance? In this section, we theoretically justify that the Sconf learning framework is robust to noise on similarity confidence by bounding the estimation error of learning with noisy confidence.

Suppose we are given the noisy Sconf data pairs: $\bar{\mathcal{S}}_{n}=\{(\bm{x}_{i},\bm{x}_{i}'),\bar{s}_{i}\}_{i=1}^{n}$, where $\bar{s}_{i}$ is the noisy similarity confidence and is not necessary equal to $s(\bm{x}_{i},\bm{x}_{i}')$ (in fact, it can take the form of any real number in $[0,~1]$). For simplicity, we replace the accurate confidence $\{s_{i}\}_{i=1}^{n}$ in Eq.~(\ref{emp}) with noisy ones $\{\bar{s}_{i}\}_{i=1}^{n}$ and denote the noisy empirical risk with $\bar{R}(g)$. The minimizer of noisy risk is ${\textstyle\bar{g}=\mbox{arg}\min\limits_{g\in\mathcal{G}}\Bar{R}(g)}$. To quantify the influence of noisy similarity confidence, we deduce the following estimation error bound:
\begin{theorem}
\label{noise}
For any $\delta>0$, the following inequality holds with probability at least $1-\delta$:
\begin{align}
\textstyle R(\bar{g})\!-\!R(g^{*})\leq\frac{4L_{\ell}\mathfrak{R}_{n}(\mathcal{G})}{|\pi_{+}-\pi_{-}|}\!+\!\frac{4C_{\ell}}{|\pi_{+}-\pi_{-}|}\sqrt{\frac{\ln 2/\delta}{2n}}\!+\!\frac{2C_{\ell}\sigma_{n}}{n|\pi_{+}-\pi_{-}|},\nonumber
\end{align}
where $\sigma_{n}=\sum_{i=1}^{n}|s_{i}-\bar{s}_{i}|$ is the summation of the deviation of noisy similarity confidence. 
\end{theorem}
In a straightforward way, the deduced estimation error bound demonstrates the magnitude of the influence of noisy similarity confidence: the estimation error of $\bar{g}$ is affected up to the mean absolute error of noisy confidence and the noisy confidence only has limited influence on the performance of Sconf learning. If the summation of noise $\sigma_{n}$ has a sublinear growth rate in high probability, Sconf learning can even remain consistent, which shows that our Sconf learning framework is robust to the noisy confidence. 
\section{Consistent Risk Correction}
\label{S4}
In the previous section, we showed the unbiased risk estimator that can recover the classification risk in Eq.~(\ref{OR}) from only Sconf data with rigorous consistency analysis. It is noticeable that the positive and negative counterparts of the empirical risk, i.e., $\hat{R}_{+}(g)$ and $\hat{R}_{-}(g)$, are not bounded below and can go negative, while their expectations are non-negative by definition. This contradiction can be problematic since as in previous works \cite{punn,uunn} that severe overfitting usually occurs when the empirical risk goes negative, especially when flexible models (e.g., deep models) are used. This phenomenon can be also observed in Sconf learning, 
as shown in Figure \ref{F2}. The detailed setting of optimization algorithm is provided in supplementary materials. 
\begin{figure*}[t]
\centering
\subfigure[Kuzushiji-MNIST, MLP (ReLU, d-500-500-1)]{
\includegraphics[width=0.23\textwidth]{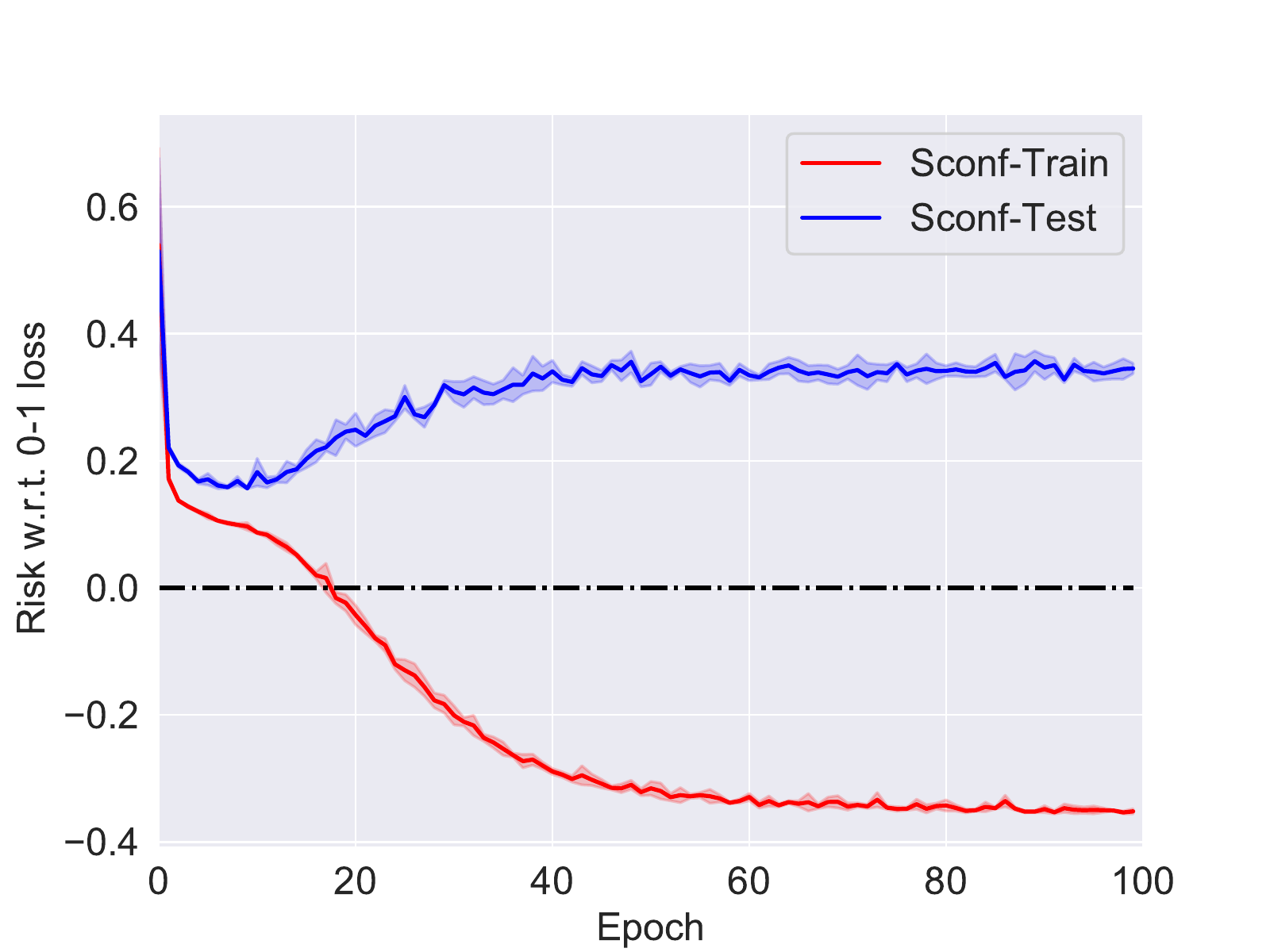}
\includegraphics[width=0.23\textwidth]{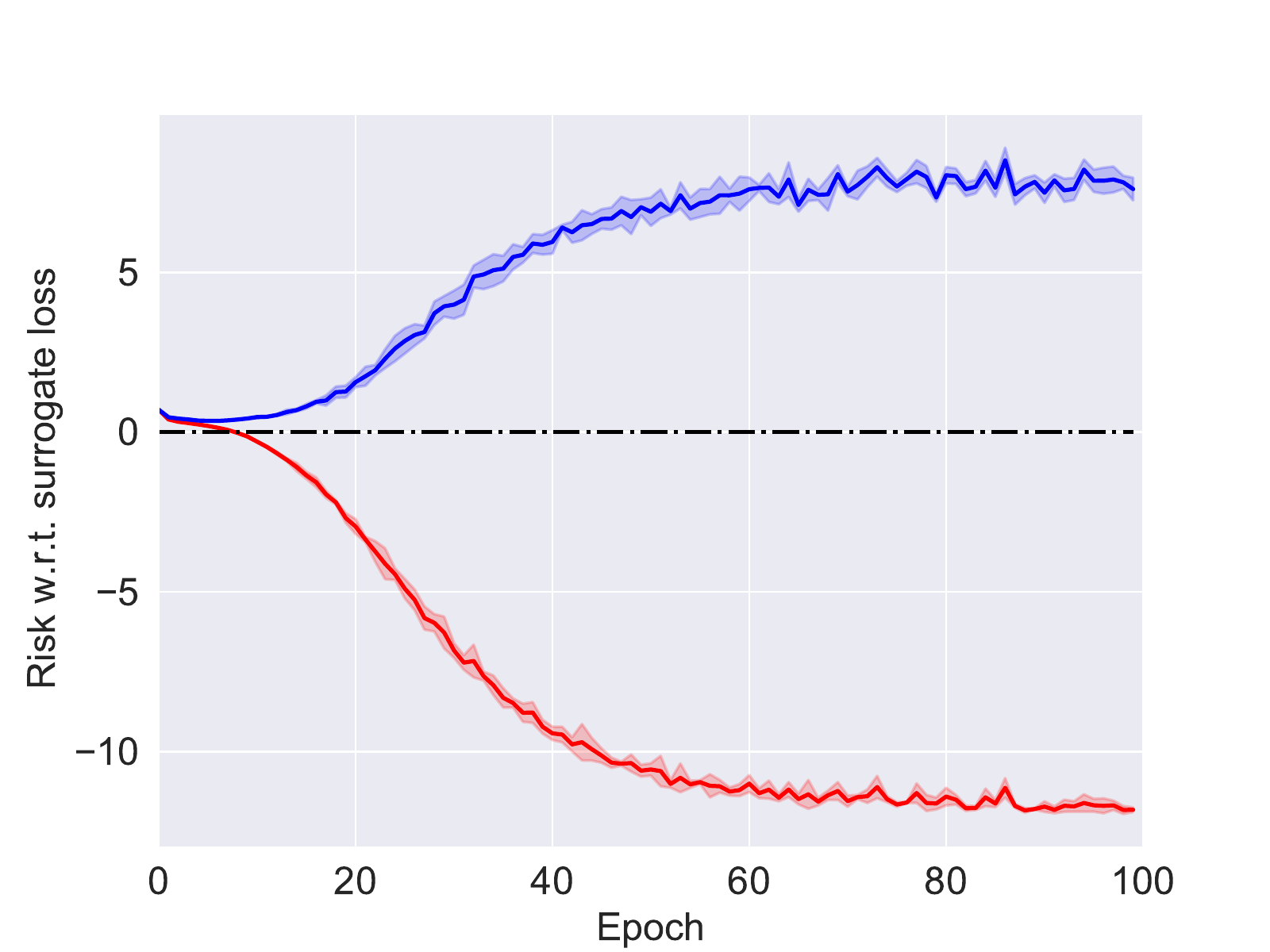}}
\subfigure[CIFAR-10, ResNet-34]{
\includegraphics[width=0.23\textwidth]{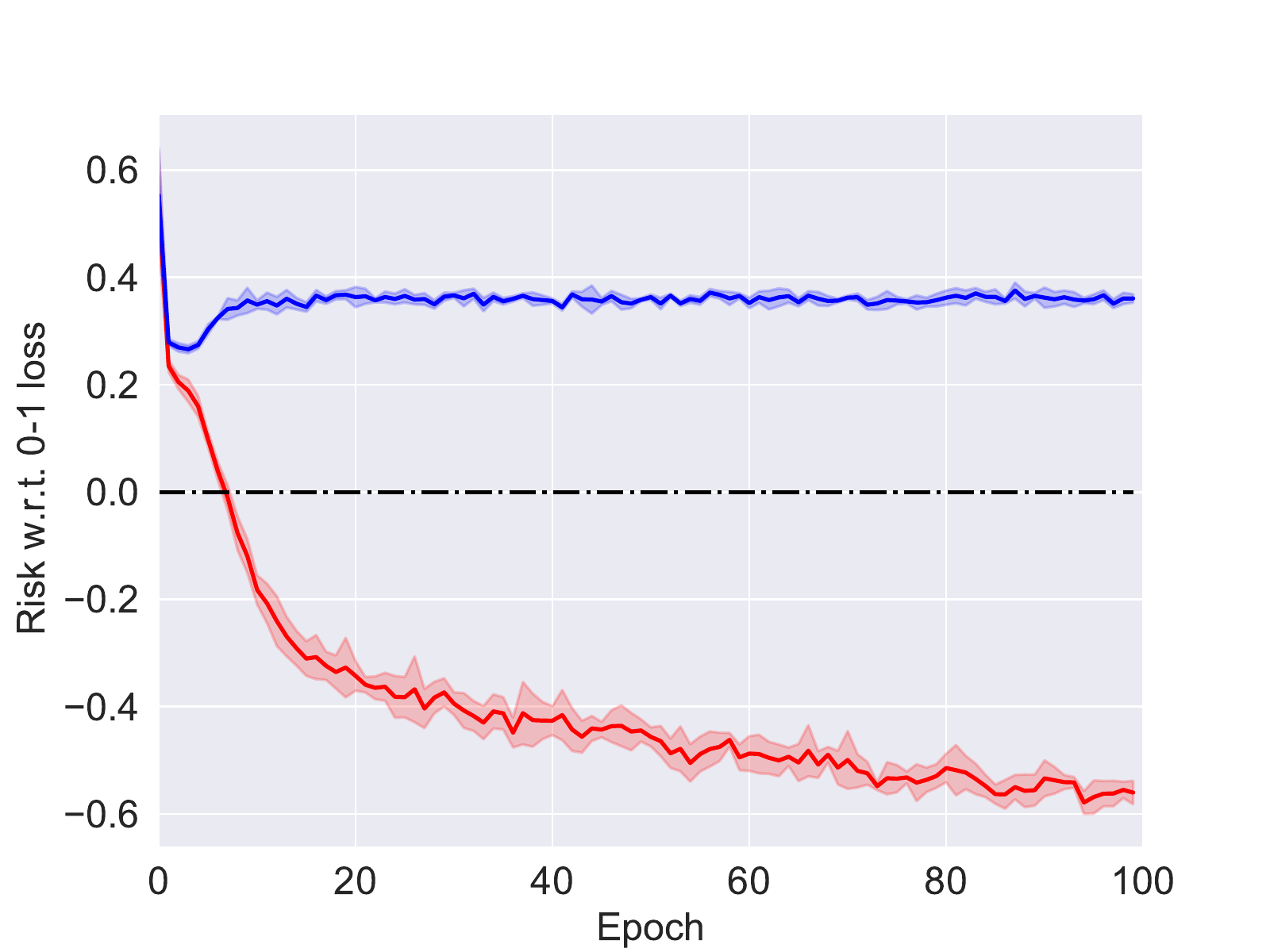}
\includegraphics[width=0.23\textwidth]{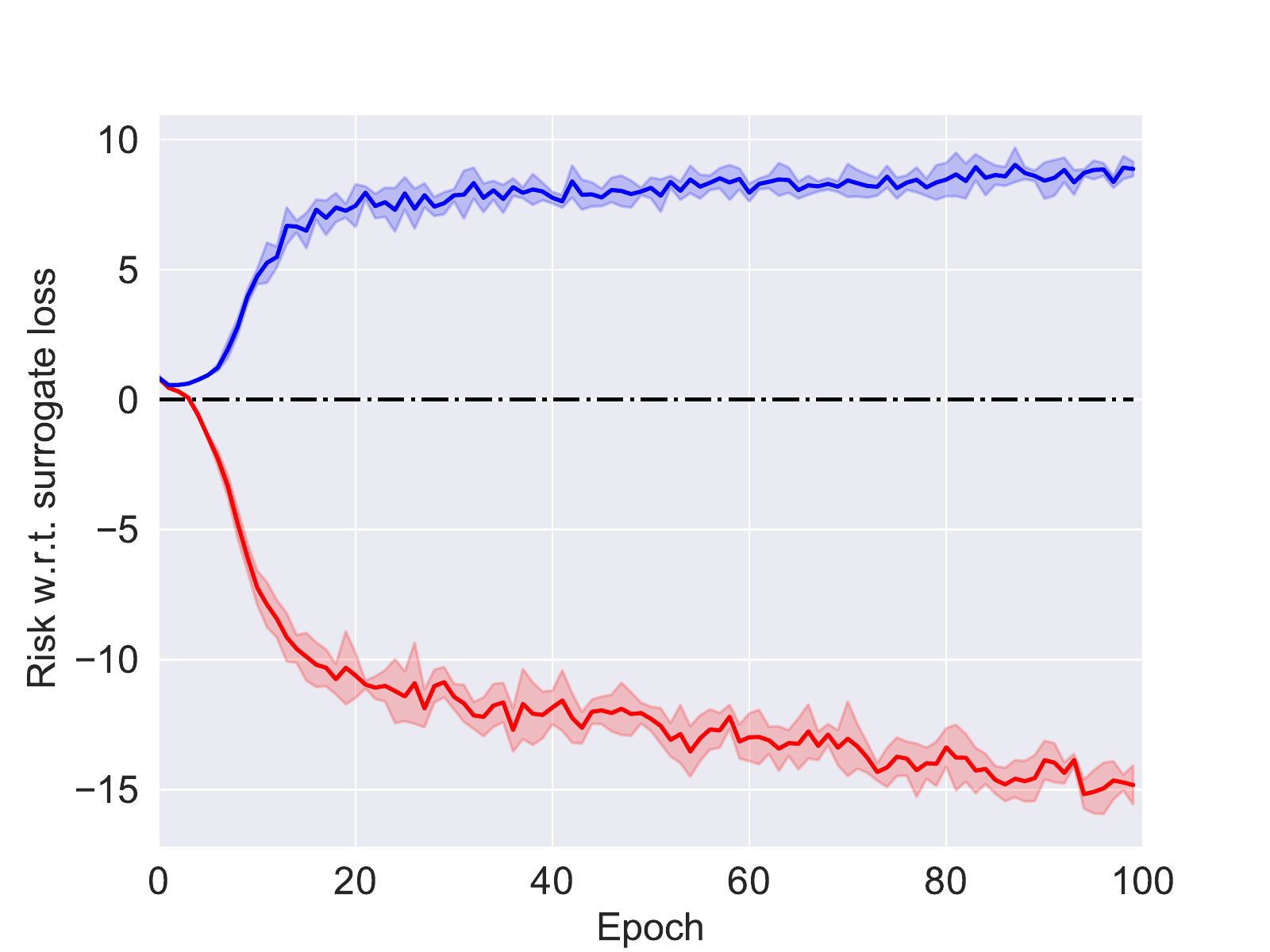}}
\begin{flushleft}
{\footnotesize{Both Kuzushiji-MNIST \cite{Kmnist} and CIFAR-10 \cite{CIFAR-10} are manually corrupted into binary classification datasets. In (a), we trained a 3-layer \textit{multi-layer perceptron} (MLP) with ReLU \cite{ReLU} on Kuzushiji-MNIST. In (b), ResNet-34 \cite{ResNet} was trained on CIFAR-10. Adam \cite{Adam} was chosen as the optimization algorithm. Logistic loss was used as the loss function. The generation of similarity confidence and the details of corrupted datasets are the same as those in Section \ref{S6}. The occurrence of overfitting and negative empirical risk is almost simultaneous: when the empirical risk (red line) goes negative, the risk on test set (blue line) stops dropping and increases rapidly.}}
\end{flushleft}
\setlength{\abovecaptionskip}{-1pt}
\setlength{\belowcaptionskip}{-10pt}
\caption{Illustration of the overfitting of unbiased Sconf learning.}
\label{F2}
\end{figure*}

In this section, we alleviate this problem with a simple yet effective \textit{risk correction} on the proposed estimator (\ref{emp}). We further show that the proposed corrected estimator can preserve its consistency by bounding its estimation error.
\subsection{General Risk Formulation}
Can we alleviate the overfitting in Sconf learning without collecting more data or changing the model? Here we give a positive answer to this question by giving a slightly modified empirical risk estimator. Since the overfitting is caused by negative empirical risk, it is a natural idea to make correction on the risk estimator when it goes negative. This idea was first proposed in \cite{punn}, where the data that yield a negative risk are ignored by applying a non-negative risk estimator. \cite{uunn} further showed that the information in those data can be helpful for generalization and should not be dropped. Based on the previous works, we propose the \textit{consistently corrected risk estimator} for Sconf learning to enforce the non-negativity:
\begin{definition}\cite{uunn}
A risk estimator $\widetilde{R}$ is called the consistently corrected risk estimator if it takes the following form:
\begin{align}
\label{ccr}
\textstyle
\widetilde{R}(g)= f\big(\hat{R}_{+}(g)\big)+f\big(\hat{R}_{-}(g)\big),
\end{align}
where $f(x)=\textstyle
\begin{cases}
x,~~~~x\geq 0,\\
k|x|,~x<0.
\end{cases}$ and $k>0$.
\end{definition}
Denote the minimizer of consistently corrected Sconf risk estimator (\ref{ccr}) with $\tilde{g}=\mbox{arg}\min_{g\in\mathcal{G}}\widetilde{R}(g)$, which can be obtained by ERM. Two representative correction functions are \textbf{N}on-\textbf{N}egative correction \cite{punn} and \textbf{ABS}olute function \cite{uunn}, with $k=0$ and 1 respectively. Their explicit formulations are shown below:
\begin{align}
\label{abs}
&\textstyle\widetilde{R}_{\mathrm{NN}}(g)=\max\left\{0,\hat{R}_{+}(g)\right\}+\max\left\{0,\hat{R}_{-}(g)\right\},\\
\label{nnn}
&\textstyle\widetilde{R}_{\mathrm{ABS}}(g)=\left|\hat{R}_{+}(g)\right|+\left|\hat{R}_{-}(g)\right|.
\end{align}

In Section \ref{S6}, we will experimentally show their efficiency in alleviating overfitting.
\subsection{Consistency Guarantee}
It is noticeable that $\widetilde{R}(g)$ is an upper bound of the unbiased risk estimator $\hat{R}(g)$ for any fixed classifier $g$, which means that $\widetilde{R}(g)$ is generally \textit{biased} and does not align with the consistency analysis in the previous section. Here we justify the use of ERM by analyzing the consistency of $\widetilde{R}(g)$ and its minimizer $\tilde{g}$. We first show that the corrected estimator is consistent and the bias decays exponentially.
\begin{theorem}\label{CRE}(Consistency of $\widetilde{R}(g)$) Assume that there are $\alpha>0$ and $\beta>0$ such that $R_{+}(g)\geq\alpha$ and $R_{-}(g)\geq\beta$. According to the assumptions in Theorem \ref{bound}, the bias of $\widetilde{R}(g)$ decays exponentially as $n\rightarrow\infty$:
$$\mathbb{E}[\widetilde{R}(g)]-R(g)\leq\textstyle\frac{(L_{f}+1)C_{\ell}}{|\pi_{+}-\pi_{-}|}\exp\left(-\frac{(\pi_{+}-\pi_{-})^{2}n}{2C_{\ell}^{2}}\right)\Delta.$$
where $\Delta=\exp(\alpha^{2})+\exp(\beta^{2})$ and $L_{f}=\max\{1,k\}$ is the Lipschitz constant of $f(\cdot)$. Furthermore, with probability at least $1-\delta$:
\begin{align*}
|\widetilde{R}(g)-R(g)|\leq\textstyle\frac{L_{\ell}C_{\ell}}{|\pi_{+}-\pi_{-}|}\sqrt{\frac{\ln2/\delta}{2n}}+\frac{(L_{f}+1)C_{\ell}}{|\pi_{+}-\pi_{-}|}\exp\left(-\frac{(\pi_{+}-\pi_{-})^{2}n}{2C_{\ell}^{2}}\right)\Delta.
\end{align*}
\end{theorem}
Based on Theorem \ref{CRE}, we show that the empirical risk minimizer $\tilde{g}$ obtained by ERM converges to $g^{*}$ in the same rate of $\mathcal{O}_{p}(1/\sqrt{n})$.

\begin{theorem}\label{CON}(Estimation error bound of $\tilde{g}$) Based on the assumptions and notations above, with probability at least $1-\delta$:
\begin{align*}
R(\tilde{g})-R(g^{*})\leq\textstyle\frac{2(L_{f}+1)C_{\ell}}{|\pi_{+}-\pi_{-}|}\exp\left(-\frac{(\pi_{+}-\pi_{-})^{2}n}{2C_{\ell}^{2}}\right)\Delta+\frac{2L_{\ell}\mathfrak{R}_{n}(\mathcal{G})}{|\pi_{+}-\pi_{-}|}+\sqrt{\frac{\ln 6/\delta}{2n}}\left(\frac{2(L_{\ell}+1)C_{\ell}}{|\pi_{+}-\pi_{-}|}\right).
\end{align*}
\end{theorem}
Theorem \ref{CON} shows that learning with $\widetilde{R}(g)$ is also consistent and it has the same convergence rate as learning with $R(g)$ since the additional exponential term is of lower order.
\section{Experiments}\label{S6}

\begin{table*}[t]
\caption{Mean and standard deviation of the classification accuracy over 5 trails in percentage with linear models on various synthetic datasets. Std. is the standard deviation of Gaussian noise. The best and comparable methods based on the paired t-test at the significance level 5\% are highlighted in boldface.}
\setlength{\belowcaptionskip}{-0.1pt}
\label{TT}
\centering
\resizebox{1.00\textwidth}{!}{
\setlength{\tabcolsep}{7mm}{
\begin{tabular}{c|cccc|c}
\toprule
Setup&Sconf&Sconf (std. = 0.1)&Sconf (std. = 0.2)&Sconf (std. = 0.3)&Supervised\\
\midrule
A&$\bm{89.91\pm0.18}$&$89.84\pm0.03$&$\bm{89.55\pm0.64}$&$89.64\pm0.33$&$89.66\pm0.41$\\
B&$\bm{90.62\pm0.28}$&$\bm{90.34\pm0.40}$&$90.03\pm0.33$&$89.49\pm0.92$&$90.71\pm0.17$\\
C&$\bm{88.05\pm0.30}$&$\bm{88.14\pm0.14}$&$\bm{87.92\pm0.57}$&$\bm{87.91\pm0.38}$&$88.14\pm0.16$\\
D&$\bm{90.43\pm0.14}$&$\bm{90.29\pm0.30}$&$\bm{90.40\pm0.15}$&$90.20\pm0.14$&$90.56\pm0.16$\\
\bottomrule
\end{tabular}
}
}
\end{table*}
\vspace{-8pt}
\begin{figure*}[t]
\centering
\subfigure[Setup A]{\includegraphics[width=0.24\textwidth]{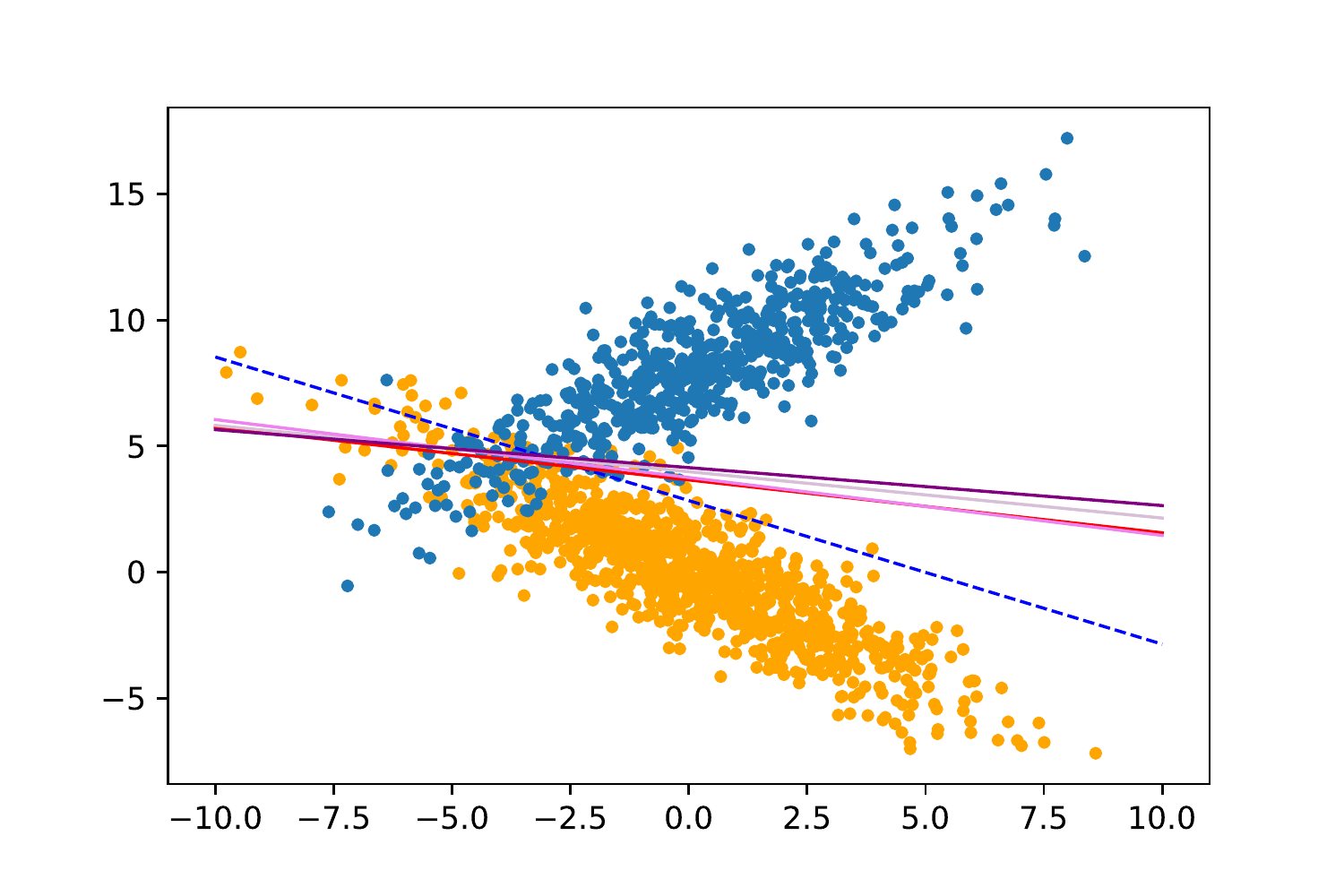}}
\subfigure[Setup B]{\includegraphics[width=0.24\textwidth]{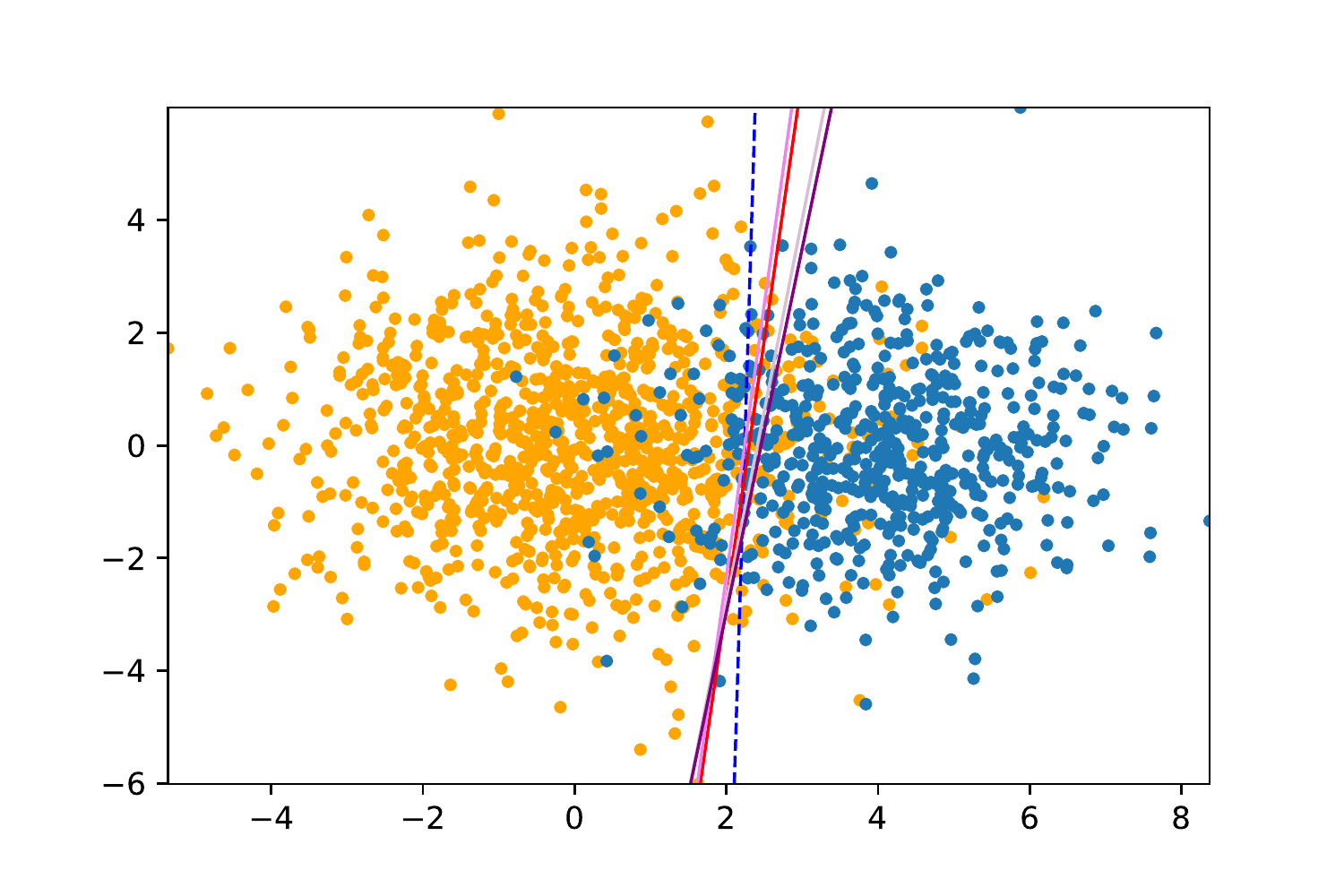}}
\subfigure[Setup C]{\includegraphics[width=0.24\textwidth]{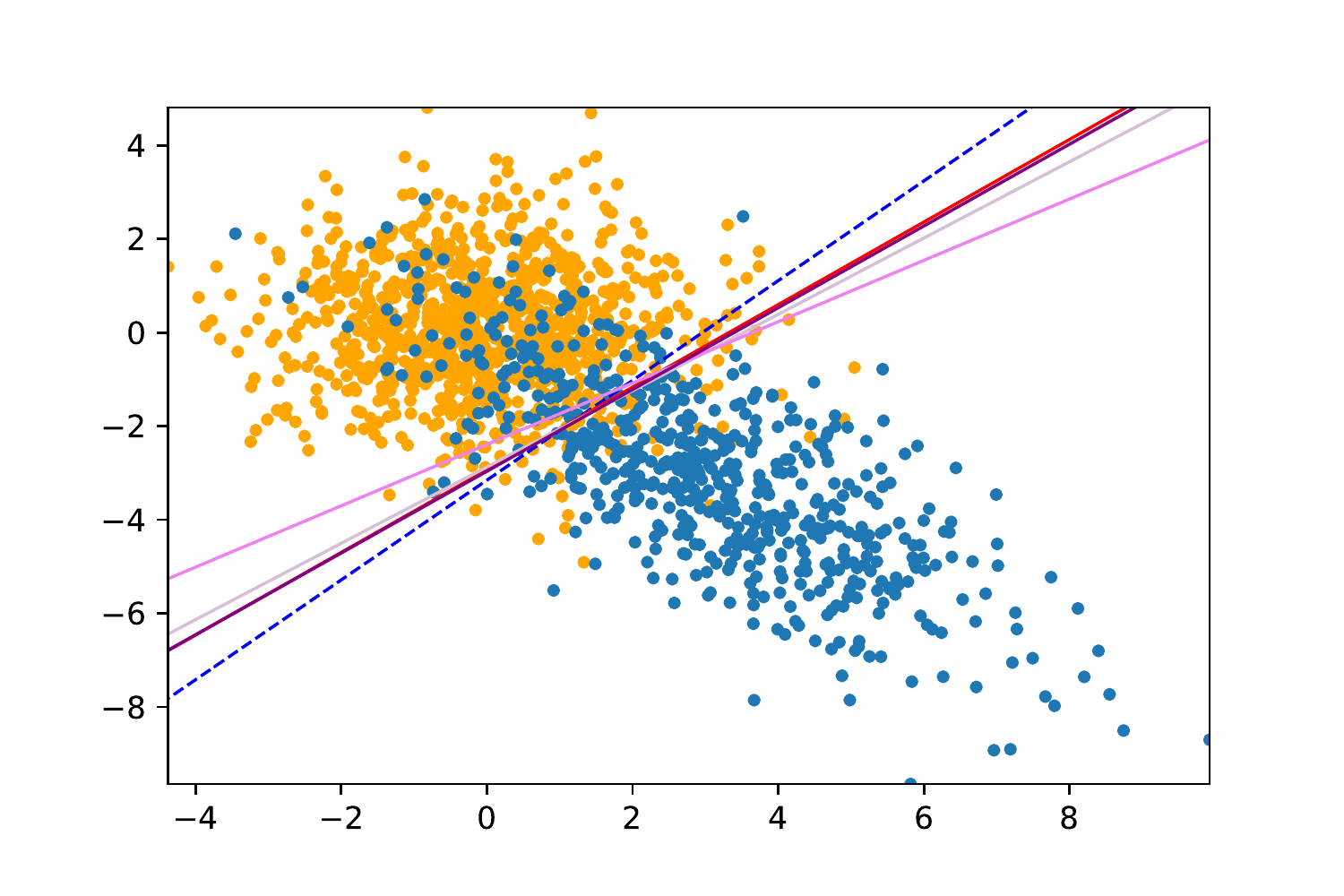}}
\subfigure[Setup D]{\includegraphics[width=0.24\textwidth]{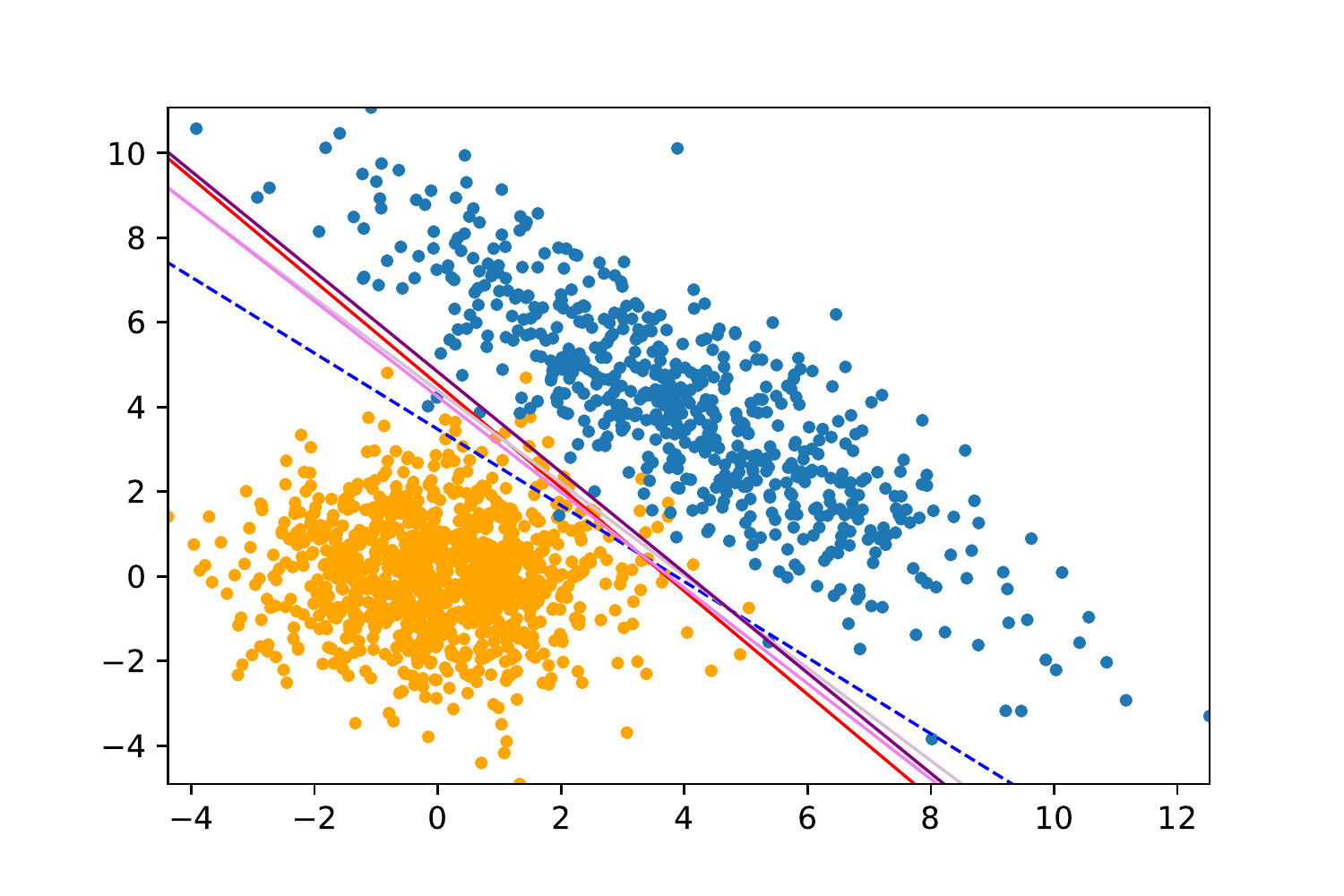}}
\includegraphics[width=0.6\textwidth, trim=0 0 10 0]{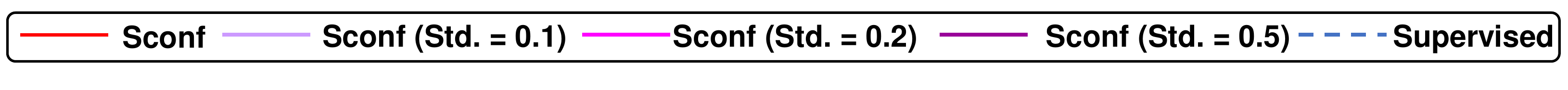}
\caption{Illustration of Sconf learning on different scales of noise and Gaussian distributions on a single trail.}
\label{FF}
\end{figure*}

\begin{table*}[t]
\centering
\caption{Mean and standard deviation of the classification accuracy over 5 trials in percentage with deep models. The best and comparable methods based on the paired t-test at the significance level 5\% are highlighted in boldface.}
\label{T22}
\resizebox{1.00\textwidth}{!}{
\begin{tabular}{ccccccc}
\toprule
\multirowcell{2}{Datasets}&\multicolumn{3}{c}{Proposed}&\multicolumn{3}{c}{Baselines}\\ \cmidrule(lr){2-4}\cmidrule(lr){5-7} &Sconf-Unbiased&Sconf-ABS&Sconf-NN&SD&Siamese&Contrastive
\\
\midrule
MNIST&$87.22\pm 2.11$&$\bm{96.12\pm 2.31}$&$\bm{96.04\pm 1.23}$&$86.57\pm 0.78$&$55.08\pm 3.94$&$71.91\pm 2.39$ \\
Kuzushiji-MNIST&$78.12\pm 3.08$&$\bm{89.25\pm 1.58}$&$\bm{90.00\pm 0.55}$&$76.42\pm 4.09$&$59.82\pm 6.15$&$67.18\pm 5.41$\\
Fashion-MNIST&$86.28\pm 7.03$&$\bm{91.44\pm 0.39}$&$\bm{91.37\pm 0.30}$&$83.61\pm 8.94$&$58.29\pm 4.42$ &$64.97\pm 5.76$ \\
EMNIST-Digits&$87.96\pm1.67$&$\bm{96.62\pm 0.06}$&$96.21\pm 0.11$&$76.18\pm 11.21$&$53.08\pm 2.55$ &$66.37\pm 5.40$ \\
EMNIST-Letters&$77.14\pm3.71$&$\bm{86.32\pm 1.20}$&$\bm{86.72\pm 1.39}$&$76.18\pm 11.21$&$55.76\pm 3.95$ &$60.29\pm 3.14$ \\
EMNIST-Balanced&$68.61\pm10.21$&$\bm{74.94\pm 2.92}$&$\bm{74.83\pm 3.40}$&$64.03\pm 14.66$&$52.61\pm 1.22$ &$58.30\pm 2.19$\\
CIFAR-10&$65.68\pm 5.03$&\bm{$84.71\pm 1.41$}&\bm{$84.49\pm 1.14$}&$60.39\pm 6.56$&$59.83 \pm 2.75$ & $54.38\pm 1.48$\\
SVHN&$72.88\pm 3.15$&\bm{$83.51\pm 0.65$}&$82.37\pm 0.23$&$71.48\pm 5.43$&$60.90\pm 5.01$ &$69.26\pm 2.97$ \\
\bottomrule
\end{tabular}
}
\end{table*}

\begin{figure*}[t]
\centering
\subfigure[{\scriptsize MNIST}]{\includegraphics[width=0.23\textwidth]{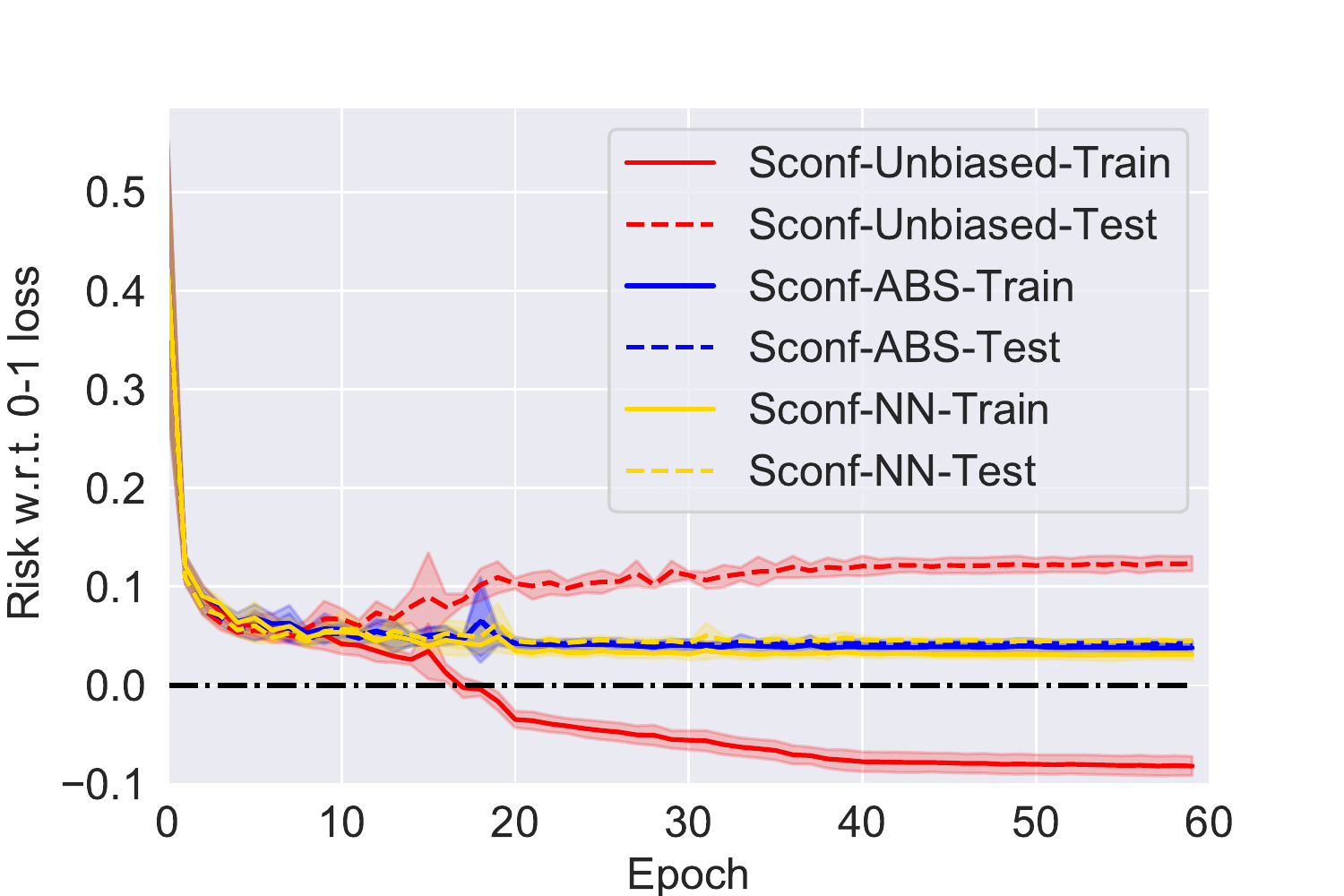}}
\subfigure[{\scriptsize Kuzushiji-MNIST}]{\includegraphics[width=0.23\textwidth]{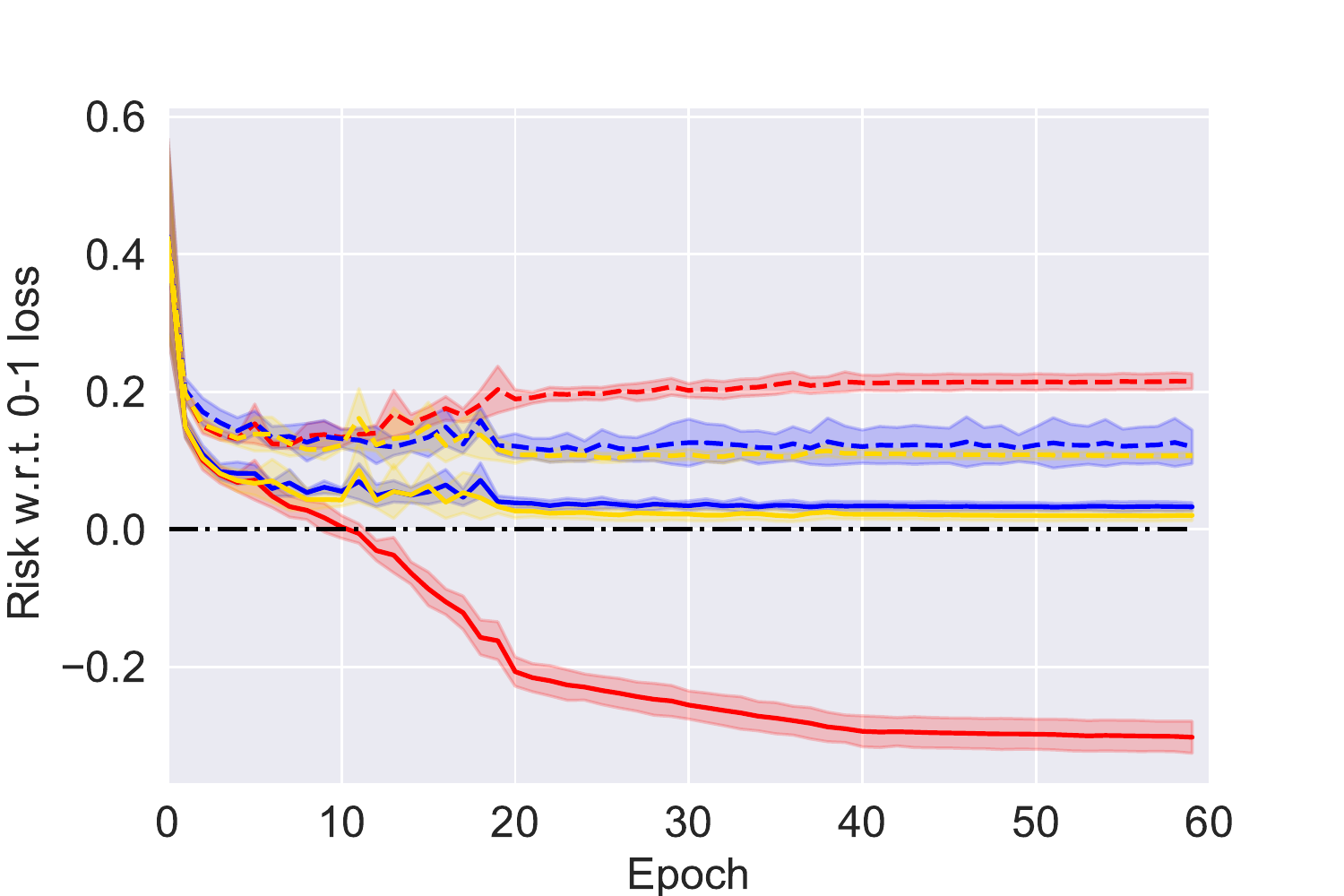}}
\subfigure[{\scriptsize Fashion-MNIST}]{\includegraphics[width=0.23\textwidth]{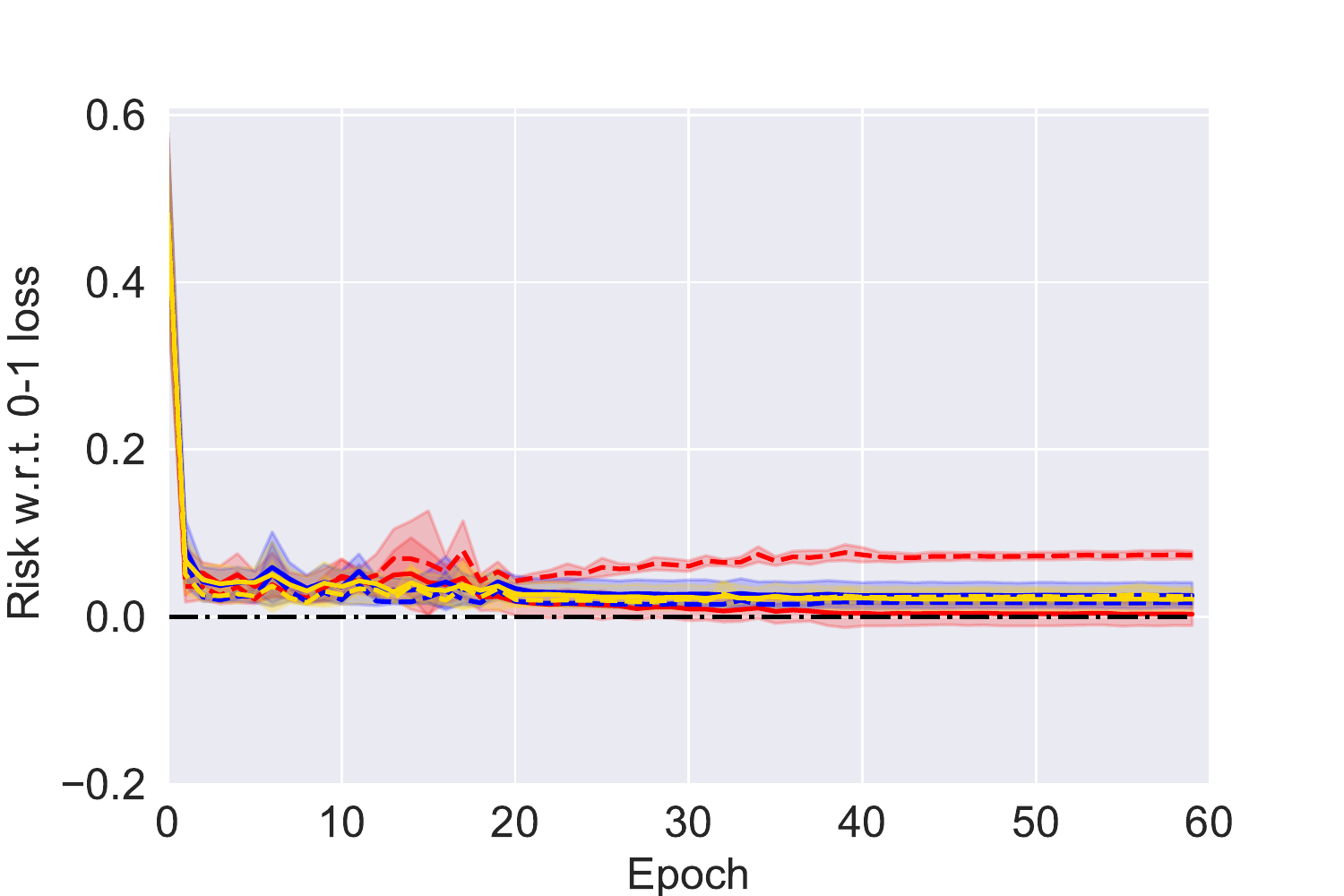}}
\subfigure[{\scriptsize{EMNIST-Digits}}]{\includegraphics[width=0.23\textwidth]{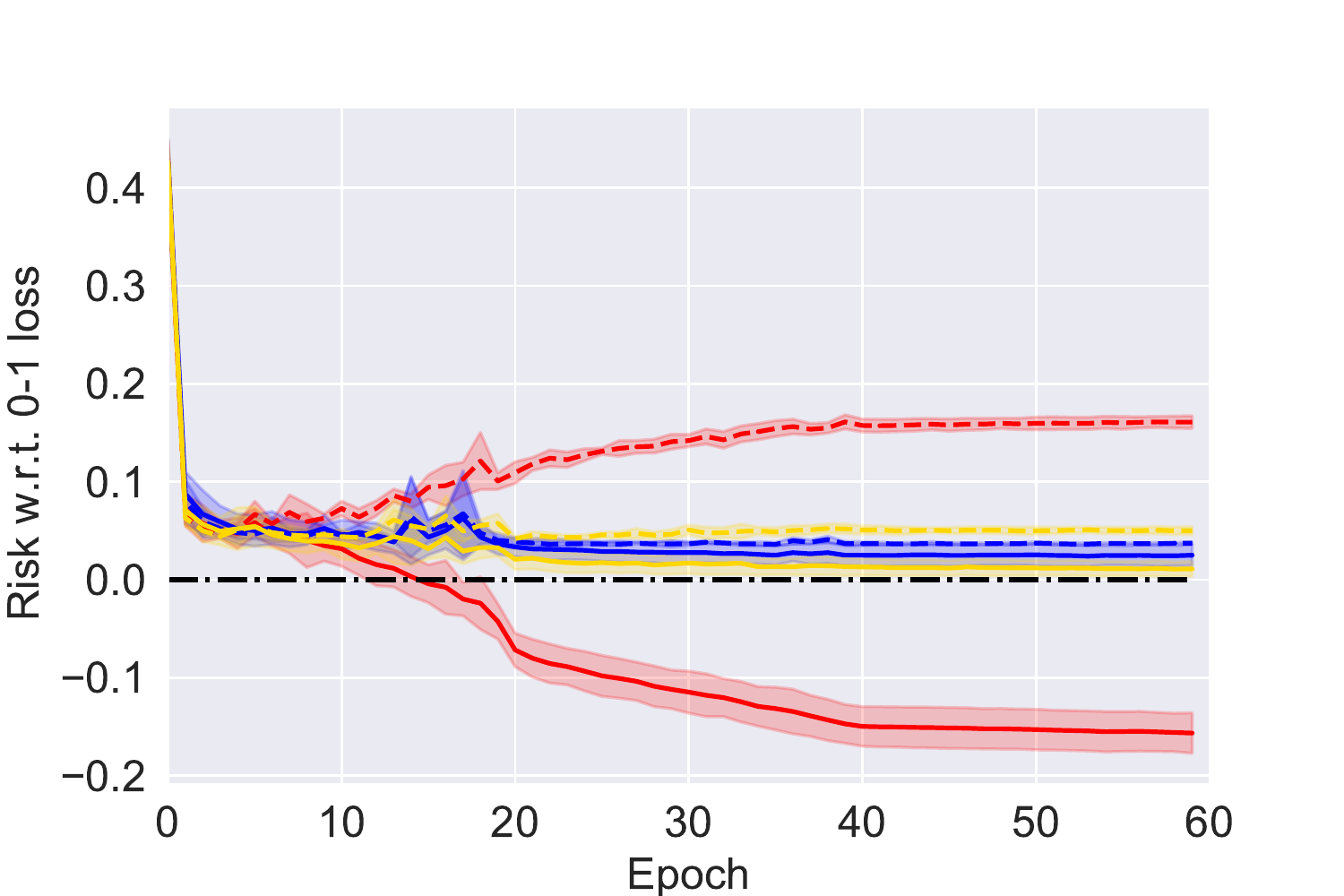}}\\
\vspace{-10pt}
\ \ \subfigure[{\scriptsize{EMNIST-Letters}}]{\includegraphics[width=0.23\textwidth]{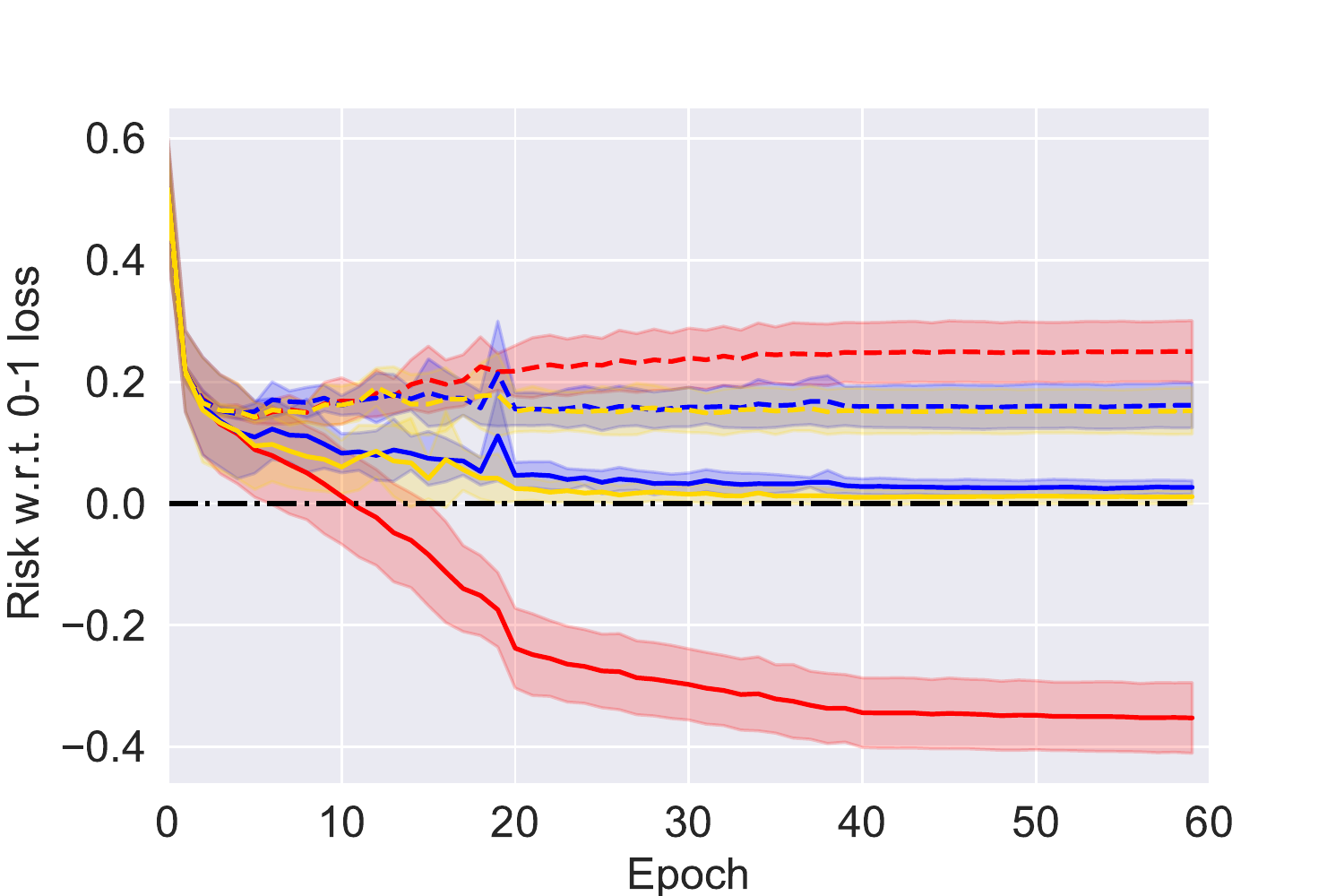}}
\subfigure[{\scriptsize{EMNIST-Balanced}}]{\includegraphics[width=0.23\textwidth]{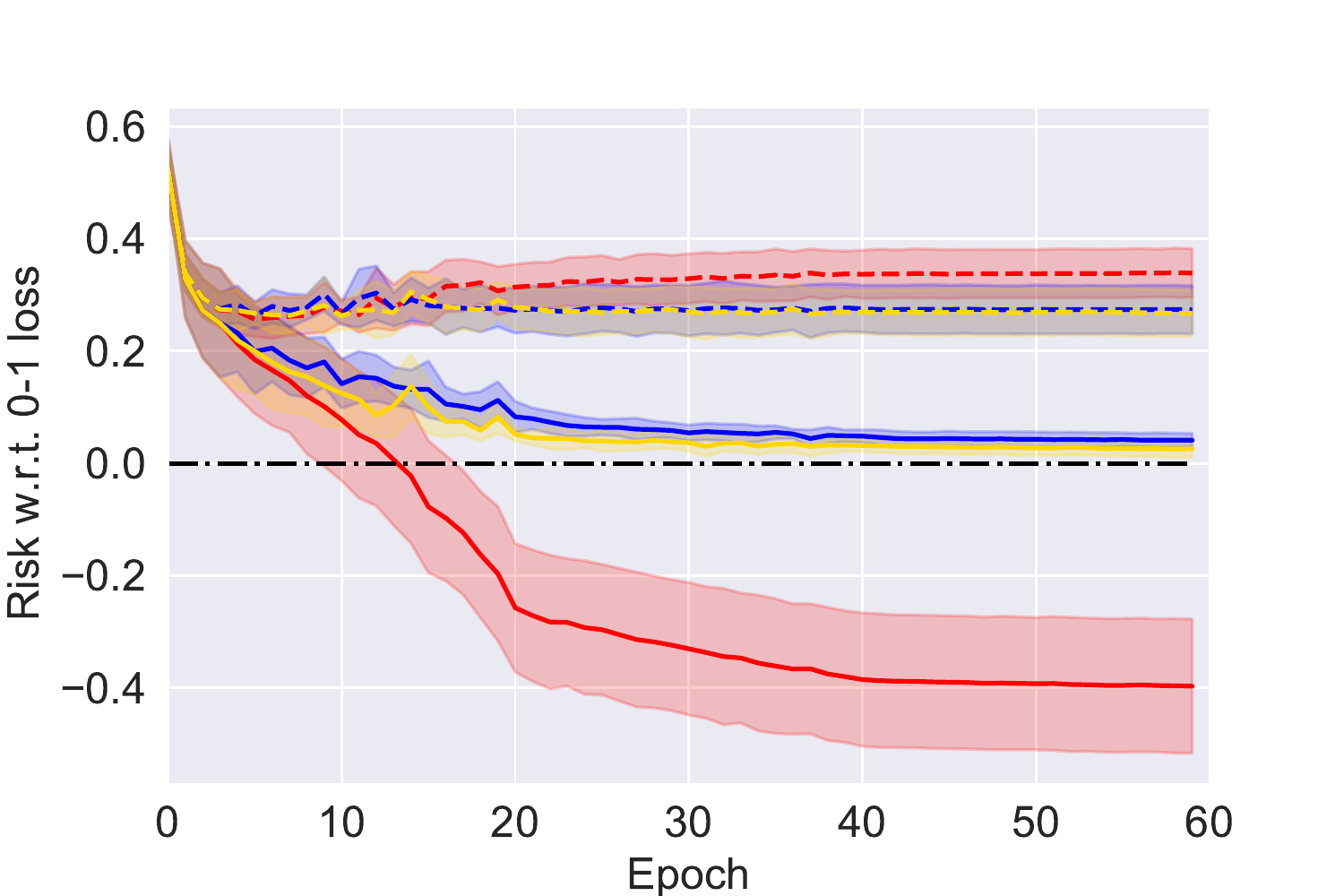}}
\subfigure[{\scriptsize SVHN}]{
\includegraphics[width=0.23\textwidth]{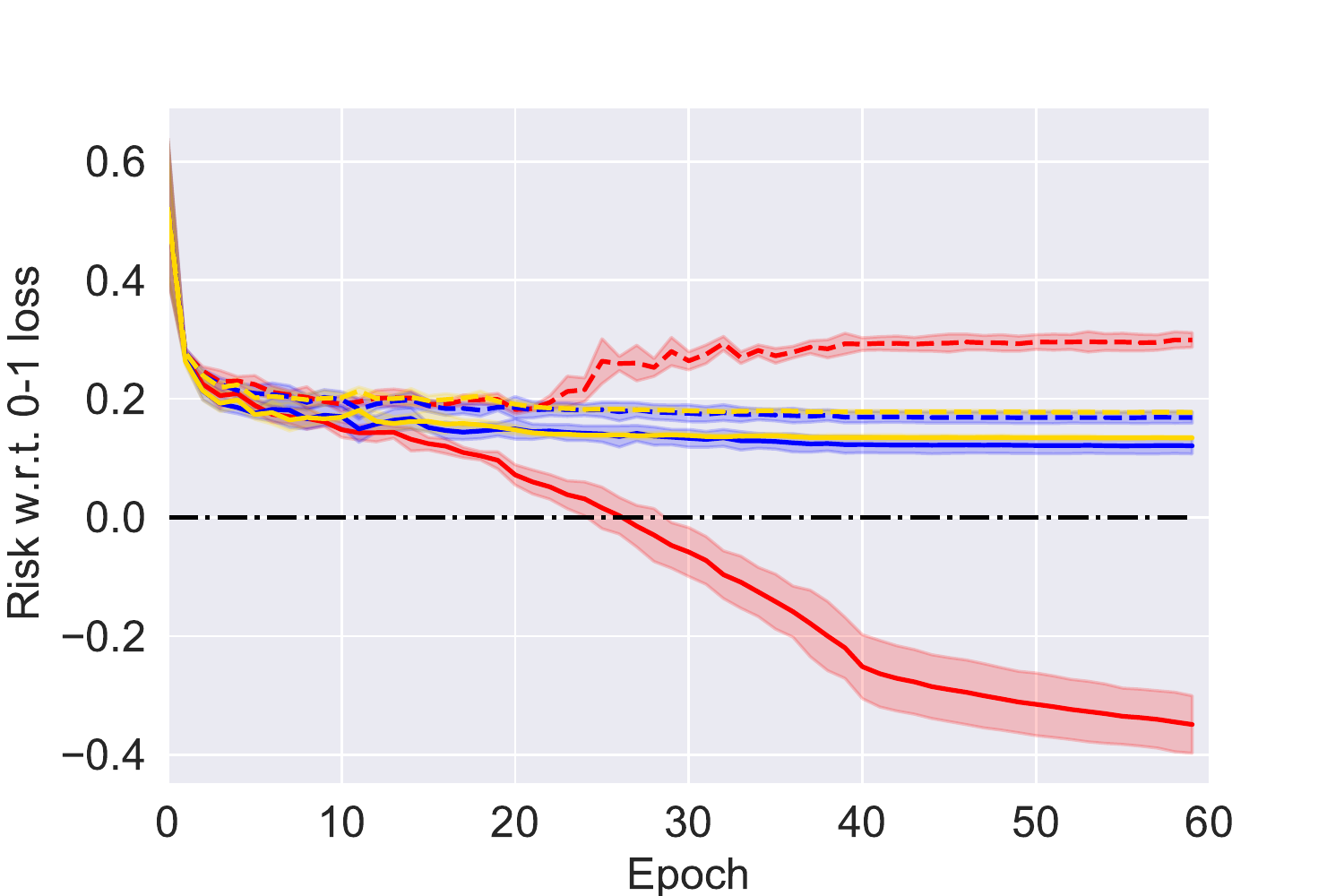}}
\subfigure[{\scriptsize CIFAR-10} ]{\includegraphics[width=0.23\textwidth]{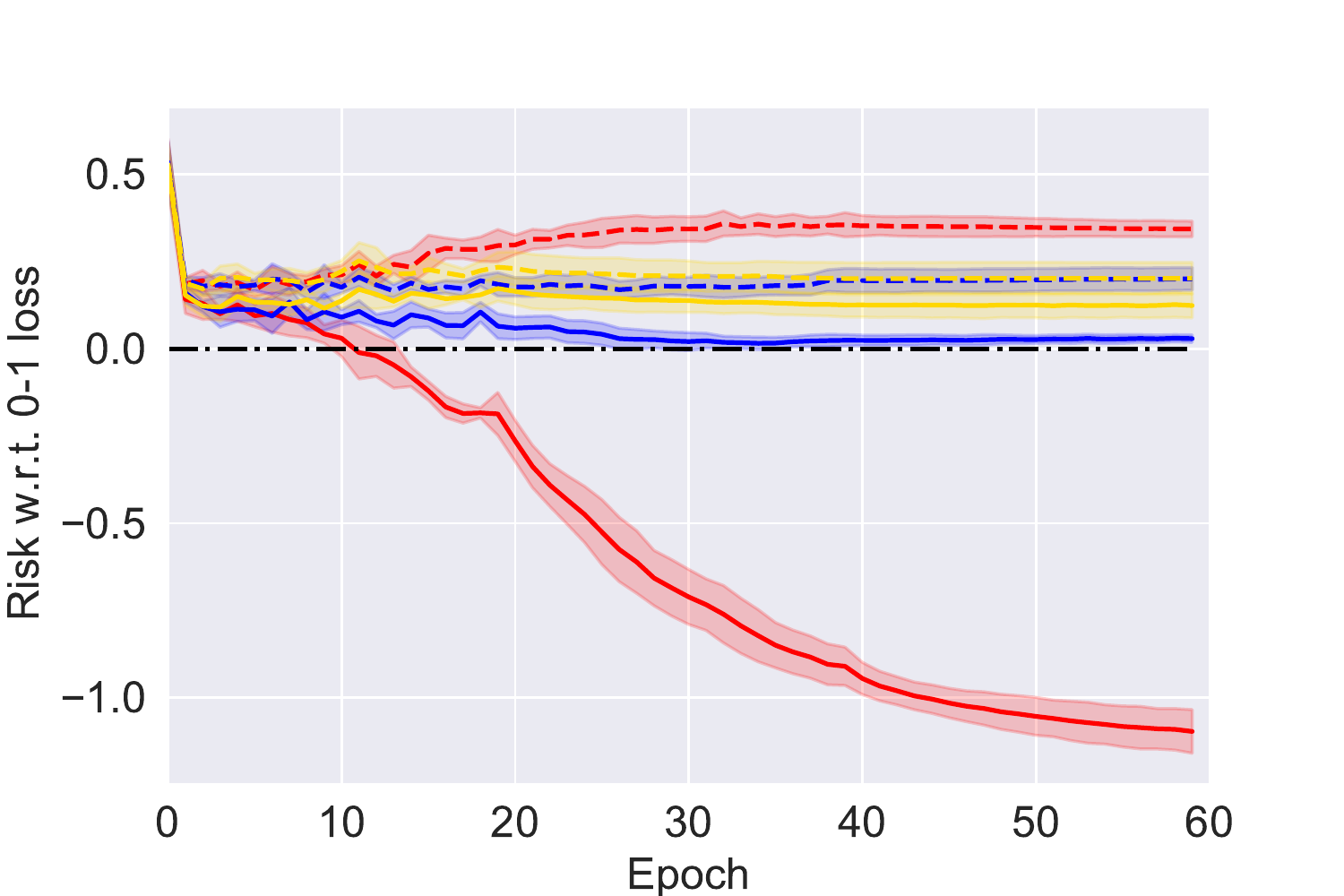}}
\caption{Experimental Results of proposed methods. Dark colors show the mean accuracy and light colors show the standard deviation.}
\label{F33}
\end{figure*}
In this section, we demonstrate the usefulness of proposed methods on both synthetic and benchmark datasets with data generation process in Section \ref{S3}. Sconf-Unbiased, Sconf-ABS and Sconf-NN are short for ERM with risk estimators in Eqs. (\ref{emp}), (\ref{abs}), and (\ref{nnn}), respectively.
\subsection{Synthetic Experiments}
\label{SE}
We experimentally characterize the behaviour of Sconf learning and show its robustness to noisy confidence on the synthetic datasets.

\textbf{Setup:} We generated the positive and negative data according to the 2-dimensional Gaussian distributions with different means and covariance for $p_{+}(\bm{x})$ and $p_{-}(\bm{x})$. The setups of data generation distributions are provided in the supplementary material. 

500 positive data and 300 negative data were generated independently from each distribution for training. We dropped the class labels for Sconf learning and generated the unlabeled data pair according to the data generation process in Section \ref{S3}. Then we analytically computed the class posterior probability $r_{+}(\bm{x})$ from the two Gaussian densities and equipped the unlabeled data pairs with true similarity confidence, which was obtained based on Lemma \ref{L1}. 1000 positive data and 600 negative data were generated in the same way for testing. 

The linear-in-input model $f(\bm{x})=\bm{\omega}^{\top}\bm{x}+b$ and logistic loss were used. We trained the model with Adam for 100 epochs (full-batch size) and default momentum parameter. The learning rate was initially set to 0.1 and divided by 10 every 30 epochs. To generate noisy similarity confidence, we added zero-mean Gaussian noise with different scales of standard deviation chosen from $\{0.1~,0.2,~0.3\}$ on the obtained similarity confidence. When the noisy similarity confidence was over 1 or below 0, we clipped it to 1 or rounded up to 0, respectively. The results of fully supervised learning is also provided.

\textbf{Experimental results:} The results are shown in Table \ref{TT} and Figure \ref{FF}. Compared with fully supervised learning, Sconf learning has similar accuracy on all synthetic datasets. The decline in performance under different scales of noise is not significant, which shows the robustness of Sconf learning.
\subsection{Benchmark Experiments}
\label{bench}
Here we conducted experiments with deep neural networks on the more realistic benchmark datasets.

\textbf{Datasets:} We evaluated the performance of proposed methods on six widely-used benchmarks MNIST \cite{Mnist}, Fashion-MNIST \cite{Fmnist}, Kuzushiji-MNIST \cite{Kmnist}, EMNIST \cite{Emnist}, SVHN \cite{SVHN}, and CIFAR-10 \cite{CIFAR-10}. Following \cite{uunn}, we manually corrupted the multi-class datasets into binary classification datasets. The detailed statistics of datasets are in the supplementary materials. 

\textbf{Baselines:} We compared our methods with both statistical learning-based and representation learning-based similarity learning baselines, including similarity-dissimilarity learning (SD) \cite{SD}, Siamese network \cite{Siamese}, and contrastive loss \cite{Contrast}. Since we can only get the vector representation rather than class label using Siamese network and contrastive loss, we adopted the one-shot setting in \cite{Siamese} and randomly chose two samples with different labels as the prototypes. Then prediction is determined according to the similarity between each test instance and the prototypes.

\textbf{Experimental setup:} We trained the proposed methods and baseline methods with the same model on a certain dataset with logistic loss. Different models are used on different datasets as summarized in Figure \ref{F33}. Since Siamese network and contrastive loss is representation learning-based, the output dimensions for them were changed to 300 on MNIST, Kuzushiji-MNIST, Fashion-MNIST, EMNIST and further increased to 1000 on CIFAR-10 and SVHN.  

For all the methods, the optimization algorithm was chosen to be Adam \cite{Adam} and the detailed setting is listed in supplementary materials. For ERM-based methods: Sconf-Unbiased, Sconf-ABS, Sconf-NN, and SD, the validation accuracy was also calculated according to their empirical risk estimators on a validation set consisted of Sconf data, which means that we do not have to collect additional ordinarily labeled data for validation when using ERM-based methods.

To simulate real similarity confidence, We generated the class posterior probability $p(y\!=\!+1|\bm{x})$ using logistic regression with the same network for each dataset, and obtained the similarity confidence according to Lemma \ref{L1}. Since the baseline methods requires data with hard similarity labels, we generated the similarity label for each data pair according to the Bernoulli distribution determined by their similarity confidence. Note that we ask the labelers for similarity confidence values in real-world Sconf learning, but we generated them through a probability classifier here. The class labels are only used for estimating the similarity confidence and the test sets are not used in any process of experiments except for testing steps.

We implemented all the methods by Pytorch \cite{Pytorch}, and conducted the experiments on NVIDIA Tesla P4 GPUs. Experimental results are reported in Figure \ref{F33} and Table \ref{T22}.

\textbf{Experimental results:} It can be observed from Table \ref{T22} that the proposed methods: Sconf-Unbiased, Sconf-ABS, and Sconf-NN outperformed the baseline methods on all the datasets. Among all the methods, Siamese network and contrastive loss performed poorly since they are representation learning-based rather than classification-oriented. Though the goal of SD learning is classification, it failed to compete with the proposed methods because it can only utilize the similarity labels degenerating from similarity confidence, which can cause the loss of supervision information. 

The efficiency of the risk correction schemes on mitigating overfitting is illustrated in Figure \ref{F33} and Table \ref{T22}. For Sconf-Unbiased, the learning curves in Figure \ref{F33} show that when the empirical risk of Sconf-Unbiased (red full line) goes negative, the test loss increases rapidly, which indicates the occurrence of overfitting. As a consequence, the performance of Sconf-Unbiased became catastrophic compared with Sconf-ABS and Sconf-NN as shown in Table \ref{T22}. Thanks to the risk correction schemes that enforce the non-negativity of empirical risk, Sconf-ABS and Sconf-NN did not suffer from the overfitting caused by negative empirical risk and greatly outperformed Sconf-Unbiased and all the baseline methods. For Sconf-ABS and Sconf-NN, the learning curves of their test and training loss are consistent, i.e., the minimization of training loss corresponding to the corrected risk estimators implies the minimization of classification risk. This observation indicates that the corrected risk estimators can better represent the classification risk compared with the unbiased risk estimator.
\section{Conclusion}
We proposed a novel weakly supervised learning setting and effective algorithms for learning from unlabeled data pairs equipped with similarity confidence, where no class labels or similarity labels are needed. We proposed the unbiased risk estimator from unlabeled data pairs with similarity confidence and further improved its performance against overfitting via a risk correction scheme. Furthermore, we proved the consistency of the minimizers of the risk estimator and corrected risk estimators. Experimental results showed that the proposed methods outperform baseline methods and our proposed risk correction scheme can effectively mitigate overfitting caused by negative empirical risk.

\paragraph{Acknowledgments} This work was supported in part by the National Natural Science Foundation of China (Nos. 12071475, 11671010) and Beijing Natural Science Foundation (No. 4172035). GN and MS were supported by JST AIP Acceleration Research Grant Number JPMJCR20U3, Japan. MS was also supported by the Institute for AI and
Beyond, UTokyo.

\begin{appendix}
\section{Proof of Lemma \ref{L1}}
\label{A1}
\begin{proof}
\begin{align*}
s(\bm{x},\bm{x}')&=p(y=y'|\bm{x},\bm{x}')\\
&=p(y=y'=+1|\bm{x},\bm{x}')+p(y=y'=-1|\bm{x},\bm{x}')\\
&=\frac{p(\bm{x}, y=+1, \bm{x}', y'=+1)+p(\bm{x}, y=-1, \bm{x}', y'=-1)}{p(\bm{x}, \bm{x}')}\\
&=\frac{p(\bm{x}, y=+1)p(\bm{x}', y'=+1)+p(\bm{x}, y=-1)p(\bm{x}', y'=-1)}{p(\bm{x})p(\bm{x}')}\\
&=\frac{\pi_{+}^{2}p(\bm{x}|y=+1)p(\bm{x}'|y'=+1)+\pi_{-}^{2}p(\bm{x}|y=-1)p(\bm{x}'|y'=-1)}{p(\bm{x})p(\bm{x}')}\\
&=\frac{\pi_{+}^{2}p_{+}(\bm{x})p_{+}(\bm{x}')+\pi_{-}^{2}p_{-}(\bm{x})p_{-}(\bm{x}')}{p(\bm{x})p(\bm{x}')}
\end{align*}
\end{proof}

\section{Proof of Theorem \ref{TS}}
\label{ALS}
We give a technical lemma before proving Theorem \ref{TS}: 
\begin{proof}
According to the independence assumption $(\bm{x},y)\perp(\bm{x}',y')$, we can immediately get the independence between $\bm{x},~\bm{x}'$ and $y,~y'$. Then the following equations hold:
\begin{align*}
p_{S}(\bm{x},\bm{x}')&=p(\bm{x},\bm{x}'|y=y')=\frac{p(\bm{x},\bm{x}',y=y')}{p(y=y')}\\
&=\frac{p(\bm{x},y=+1,\bm{x}',y'=+1)+p(\bm{x},y=-1,\bm{x}',y'=-1)}{p(y=+1)p(y'=+1)+p(y=-1)p(y'=-1)}\\
&=\frac{p(\bm{x},y=+1)p(\bm{x}',y'=+1)+p(\bm{x},y=-1)p(\bm{x}',y'=-1)}{\pi_{+}^{2}+\pi_{-}^{2}}\\
&=\frac{\pi_{+}^{2}p_{+}(\bm{x})p_{+}(\bm{x}')+\pi_{-}^{2}p_{-}(\bm{x})p_{-}(\bm{x}')}{\pi_{+}^{2}+\pi_{-}^{2}}\\
&=\frac{\pi_{+}^{2}p_{+}(\bm{x})p_{+}(\bm{x}')+\pi_{-}^{2}p_{-}(\bm{x})p_{-}(\bm{x}')}{\pi_{S}}
\end{align*}
\end{proof}
Then we can prove the Theorem \ref{TS}
\begin{proof}
\begin{align*}
&\mathbb{E}_{p_{S}(\bm{x},\bm{x}')}\left[\frac{\pi_{S}(s(\bm{x},\bm{x}')-\pi_{-})(\ell(g(\bm{x}),+1)+\ell(g(\bm{x}'),+1))}{2(\pi_{+}-\pi_{-})s(\bm{x},\bm{x}')}\right]\\
&=\int\frac{\pi_{S}(s(\bm{x},\bm{x}')-\pi_{-})(\ell(g(\bm{x}),+1)+\ell(g(\bm{x}'),+1))}{2(\pi_{+}-\pi_{-})s(\bm{x},\bm{x}')}*\frac{\pi_{+}^{2}p_{+}(\bm{x})p_{+}(\bm{x}')+\pi_{-}^{2}p_{-}(\bm{x})p_{-}(\bm{x}')}{\pi_{S}}d\bm{x}d\bm{x}'\\
&=\int\frac{(\pi_{+}^{2}p_{+}(\bm{x})p_{+}(\bm{x}')+\pi_{-}^{2}p_{-}(\bm{x})p_{-}(\bm{x}')-\pi_{-}p(\bm{x})p(\bm{x}'))(\ell(g(\bm{x}),+1)+\ell(g(\bm{x}'),+1))}{2(\pi_{+}-\pi_{-})}d\bm{x}d\bm{x}'\\
&=\int\frac{(\pi_{+}^{2}p_{+}(\bm{x})+\pi_{-}^{2}p_{-}(\bm{x})-\pi_{-}p(\bm{x}))\ell(g(\bm{x}),+1)}{2(\pi_{+}-\pi_{-})}d\bm{x}+\int\frac{(\pi_{+}^{2}p_{+}(\bm{x}')+\pi_{-}^{2}p_{-}(\bm{x}')-\pi_{-}p(\bm{x}'))\ell(g(\bm{x}'),+1)}{2(\pi_{+}-\pi_{-})}d\bm{x}'\\
&=\int\frac{\pi_{+}(\pi_{+}-\pi_{-})p_{+}(\bm{x})\ell(g(\bm{x}),+1)}{2(\pi_{+}-\pi_{-})}d\bm{x}+\int\frac{\pi_{+}(\pi_{+}-\pi_{-})p_{+}(\bm{x})\ell(g(\bm{x}'),+1)}{2(\pi_{+}-\pi_{-})}d\bm{x}'\\
&=\int\frac{\pi_{+}p_{+}(\bm{x})\ell(g(\bm{x}),+1)}{2}d\bm{x}+\int\frac{\pi_{+}p_{+}(\bm{x})\ell(g(\bm{x}'),+1)}{2}d\bm{x}'\\
&=\frac{\pi_{+}\mathbb{E}_{+}[\ell(g(\bm{x}),+1)]}{2}+\frac{\pi_{+}\mathbb{E}_{+}[\ell(g(\bm{x}),+1)]}{2}\\
&=\pi_{+}\mathbb{E}_{+}[\ell(g(\bm{x}),+1)]
\end{align*}

Symmetrically, we have:
\begin{align*}
&\mathbb{E}_{p_{S}(\bm{x},\bm{x}')}\left[\frac{\pi_{S}(\pi_{+}-s(\bm{x},\bm{x}'))(\ell(g(\bm{x}),-1)+\ell(g(\bm{x}'),-1))}{2(\pi_{+}-\pi_{-})s(\bm{x},\bm{x}')}\right]\\
&=\int\frac{\pi_{S}(\pi_{+}-s(\bm{x},\bm{x}'))(\ell(g(\bm{x}),-1)+\ell(g(\bm{x}'),-1))}{2(\pi_{+}-\pi_{-})s(\bm{x},\bm{x}')}*\frac{\pi_{+}^{2}p_{+}(\bm{x})p_{+}(\bm{x}')+\pi_{-}^{2}p_{-}(\bm{x})p_{-}(\bm{x}')}{\pi_{S}}d\bm{x}d\bm{x}'\\
&=\int\frac{(\pi_{+}p(\bm{x})p(\bm{x}')-\pi_{+}^{2}p_{+}(\bm{x})p_{+}(\bm{x}')-\pi_{-}^{2}p_{-}(\bm{x})p_{-}(\bm{x}'))(\ell(g(\bm{x}),-1)+\ell(g(\bm{x}'),-1))}{2(\pi_{+}-\pi_{-})}d\bm{x}d\bm{x}'\\
&=\int\frac{(\pi_{+}p(\bm{x})-\pi_{+}^{2}p_{+}(\bm{x})-\pi_{-}^{2}p_{-}(\bm{x}))\ell(g(\bm{x}),-1)}{2(\pi_{+}-\pi_{-})}d\bm{x}+\int\frac{(\pi_{+}p(\bm{x}')-\pi_{+}^{2}p_{+}(\bm{x}')-\pi_{-}^{2}p_{-}(\bm{x}'))\ell(g(\bm{x}'),-1)}{2(\pi_{+}-\pi_{-})}d\bm{x}'\\
&=\int\frac{\pi_{-}(\pi_{+}-\pi_{-})p_{-}(\bm{x})\ell(g(\bm{x}),-1)}{2(\pi_{+}-\pi_{-})}d\bm{x}+\int\frac{\pi_{-}(\pi_{+}-\pi_{-})p_{-}(\bm{x})\ell(g(\bm{x}'),-1)}{2(\pi_{+}-\pi_{-})}d\bm{x}'\\
&=\int\frac{\pi_{-}p_{-}(\bm{x})\ell(g(\bm{x}),-1)}{2}d\bm{x}+\int\frac{\pi_{-}p_{-}(\bm{x})\ell(g(\bm{x}'),-1)}{2}d\bm{x}'\\
&=\frac{\pi_{-}\mathbb{E}_{-}[\ell(g(\bm{x}),-1)]}{2}+\frac{\pi_{-}\mathbb{E}_{-}[\ell(g(\bm{x}),-1)]}{2}\\
&=\pi_{-}\mathbb{E}_{-}[\ell(g(\bm{x}),-1)]
\end{align*}
Then we have:
\begin{align}
R_{\mathrm{S}}(g)=&\mathbb{E}_{p_{S}(\bm{x},\bm{x}')}\left[\frac{\pi_{S}(s(\bm{x},\bm{x}')-\pi_{-})(\ell(g(\bm{x}),+1)+\ell(g(\bm{x}'),+1))}{2(\pi_{+}-\pi_{-})s(\bm{x},\bm{x}')}\right]
\nonumber\\
&+\mathbb{E}_{p_{S}(\bm{x},\bm{x}')}\left[\frac{\pi_{S}(\pi_{+}-s(\bm{x},\bm{x}'))(\ell(g(\bm{x}),-1)+\ell(g(\bm{x}'),-1))}{2(\pi_{+}-\pi_{-})s(\bm{x},\bm{x}')}\right].
\end{align}
and we can give the unbiased estimator of classification risk according to the risk expression above:
\begin{align}
\hat{R}_{\mathrm{S}}(g)=&\pi_{S}\sum\nolimits_{i=1}^{n}\frac{(s_{i}-\pi_{-})(\ell(g(\bm{x}_{i}),+1)+\ell(g(\bm{x}'_{i}),+1))}{2n(\pi_{+}-\pi_{-})s_{i}}\nonumber+\pi_{S}\sum\nolimits_{i=1}^{n}\frac{(\pi_{+}-s_{i})(\ell(g(\bm{x}_{i}),-1)+\ell(g(\bm{x}'_{i}),-1))}{2n(\pi_{+}-\pi_{-})s_{i}}.
\end{align}
which concludes the proof. 
\end{proof}

\section{Proof of Theorem \ref{coll}}
\begin{proof}
We aim to solve the following optimization problem when conducting ERM algorithm according to Theorem \ref{TS}:
\begin{align}
\label{opt}
    \min\limits_{g\in\mathcal{G}}\pi_{S}\sum\nolimits_{i=1}^{n}\left(\frac{(s_{i}-\pi_{-})(\ell(g(\bm{x}_{i}),+1)+\ell(g(\bm{x}'_{i}),+1))}{2n(\pi_{+}-\pi_{-})s_{i}}+\frac{(\pi_{+}-s_{i})(\ell(g(\bm{x}_{i}),-1)+\ell(g(\bm{x}'_{i}),-1))}{2n(\pi_{+}-\pi_{-})s_{i}}\right).
\end{align}
Notice that since $\pi_{+}>\pi_{-}$ and $s_{i}\geq\pi_{+}$ for all $i\in[n]$, we have the following 
$$
\begin{cases}
\frac{s_{i}-\pi_{-}}{2n(\pi_{+}-\pi_{-})s_{i}}\geq 0,~i=1\cdots,n \\
\frac{\pi_{+}-s_{i}}{2n(\pi_{+}-\pi_{-})s_{i}}\leq 0,~i=1\cdots,n
\end{cases}
$$
Since 0-1 loss is used, we have the conclusion that $\ell(g(\bm{x}),y)\in[0,1]$ for any $g$, $\bm{x}$, and $y$. According to the discussion above, by setting all the $\ell(\cdot,+1)$ to 0 and $\ell(\cdot,-1)$ to 1, we can get the lower bound of (\ref{opt}):
$$(\ref{opt})\geq \sum_{i=1}^{n}\frac{\pi_{S}(\pi_{+}-s_{i})}{n(\pi_{+}-\pi_{-})s_{i}}. $$
It is obvious that such setting can be realized if we let $g(\bm{x})>0$ for all the $\bm{x}$, which means that $g$ classifies all the examples as positive.
\end{proof}
\section{Proof of Lemma \ref{LC}}
\label{A2}
Before proving the Lemma \ref{LC}, we begin with the proof of two important technical Lemmas:
\begin{lemma}
\label{AL1} For any binary loss function $\ell(\cdot, \cdot):\mathbb{R}\times\{+1, -1\}\rightarrow \mathbb{R}^{+}$:
\begin{align}
\mathbb{E}_{U^{2}}[s(\bm{x},\bm{x}')\ell(g(\bm{x}),+1)]=\pi_{+}^{2}\mathbb{E}_{+}[\ell(g(\bm{x}),+1)]+\pi_{-}^{2}\mathbb{E}_{-}[\ell(g(\bm{x}),+1)]\\
\mathbb{E}_{U^{2}}[s(\bm{x},\bm{x}')\ell(g(\bm{x}),-1)]=\pi_{+}^{2}\mathbb{E}_{+}[\ell(g(\bm{x}),-1)]+\pi_{-}^{2}\mathbb{E}_{-}[\ell(g(\bm{x}),-1)]
\end{align}
\end{lemma}
\begin{proof}
We only prove the first equation since the second one can be deduced in the same manner.
\begin{align*}
\mathbb{E}_{U^{2}}[s(\bm{x},\bm{x}')\ell(g(\bm{x}),+1)]&=\int\int\frac{\pi_{+}^{2}p_{+}(\bm{x})p_{+}(\bm{x}')+\pi_{-}^{2}p_{-}(\bm{x})p_{-}(\bm{x}')}{p(\bm{x})p(\bm{x}')}p(\bm{x})p(\bm{x}')\ell(g(\bm{x}),+1)d\bm{x}d\bm{x}'\\
&=\int\pi_{+}^{2}\ell(g(\bm{x}),+1)p_{+}(\bm{x})d\bm{x}\int p_{+}(\bm{x}')d\bm{x}'\\&~~~+\int\pi_{-}^{2}\ell(g(\bm{x}),+1)p_{-}(\bm{x})d\bm{x}\int p_{-}(\bm{x}')d\bm{x}'\\
&=\pi_{+}^{2}\int\ell(g(\bm{x}),+1)p_{+}(\bm{x})d\bm{x}+\pi_{-}^{2}\int\ell(g(\bm{x}),+1)p_{-}(\bm{x})d\bm{x}\\
&=\pi_{+}^{2}\mathbb{E}_{+}[\ell(g(\bm{x}),+1)]+\pi_{-}^{2}\mathbb{E}_{-}[\ell(g(\bm{x}),+1)]
\end{align*}
\end{proof}
\begin{lemma}
\label{AL2}
\begin{align}
\mathbb{E}_{U^{2}}[(1-s(\bm{x},\bm{x}'))\ell(g(\bm{x}),+1)]=\pi_{+}\pi_{-}\mbox{\rm\large(}\mathbb{E}_{+}[\ell(g(\bm{x}),+1)]+\mathbb{E}_{-}[\ell(g(\bm{x}),+1)]\mbox{\rm\large)}\\
\mathbb{E}_{U^{2}}[(1-s(\bm{x},\bm{x}'))\ell(g(\bm{x}),-1)]=\pi_{+}\pi_{-}\mbox{\rm\large(}\mathbb{E}_{+}[\ell(g(\bm{x}),-1)]+\mathbb{E}_{-}[\ell(g(\bm{x}),-1)]\mbox{\rm\large)}
\end{align}
\end{lemma}
\begin{proof}
First, we note that 
\begin{align*}
1-s(\bm{x},\bm{x}')&=1-\frac{\pi_{+}^{2}p_{+}(\bm{x})p_{+}(\bm{x}')+\pi_{-}^{2}p_{-}(\bm{x})p_{-}(\bm{x}')}{p(\bm{x})p(\bm{x}')}\\
&=\frac{p(\bm{x})p(\bm{x}')-(\pi_{+}^{2}p_{+}(\bm{x})p_{+}(\bm{x}')+\pi_{-}^{2}p_{-}(\bm{x})p_{-}(\bm{x}'))}{p(\bm{x})p(\bm{x}')}\\
&=\frac{(\pi_{+}p_{+}(\bm{x})+\pi_{-}p_{-}(\bm{x}))(\pi_{+}p_{+}(\bm{x}')+\pi_{-}p_{-}(\bm{x}'))-(\pi_{+}^{2}p_{+}(\bm{x})p_{+}(\bm{x}')+\pi_{-}^{2}p_{-}(\bm{x})p_{-}(\bm{x}'))}{p(\bm{x})p(\bm{x}')}\\
&=\frac{\pi_{+}\pi_{-}(p_{+}(\bm{x})p_{-}(\bm{x}')+p_{-}(\bm{x})p_{+}(\bm{x}'))}{p(\bm{x})p(\bm{x}')}.
\end{align*}
Then we can prove the first equation: 
\begin{align*}
\mathbb{E}_{U^{2}}[(1-s(\bm{x},\bm{x}'))\ell(g(\bm{x}),+1)]&=\int\int\frac{\pi_{+}\pi_{-}(p_{+}(\bm{x})p_{-}(\bm{x}')+p_{-}(\bm{x})p_{+}(\bm{x}'))}{p(\bm{x})p(\bm{x}')}p(\bm{x})p(\bm{x}')\ell(g(\bm{x}),+1)d\bm{x}d\bm{x}'\\
&=\int\pi_{+}\pi_{-}\ell(g(\bm{x}),+1)p_{+}(\bm{x})d\bm{x}\int p_{-}(\bm{x}')d\bm{x}'\\&~~~+\int\pi_{+}\pi_{-}\ell(g(\bm{x}),+1)p_{-}(\bm{x})d\bm{x}\int p_{+}(\bm{x}')d\bm{x}'\\
&=\pi_{+}\pi_{-}\left(\int\ell(g(\bm{x}),+1)p_{+}(\bm{x})d\bm{x}+\int\ell(g(\bm{x}),+1)p_{-}(\bm{x})d\bm{x}\right)\\
&=\pi_{+}\pi_{-}\mbox{\large(}\mathbb{E}_{+}[\ell(g(\bm{x}),+1)]+\mathbb{E}_{-}[\ell(g(\bm{x}),+1)]\mbox{\large)}
\end{align*}
\end{proof}
Note that similar conclusions for $\bm{x}'$ can be derived by switching $\bm{x}$ and $\bm{x}'$ in the lemmas above since they are completely symmetric.

Based on the lemmas above, we give the proof of Lemma \ref{LC}.
\begin{proof}
We first prove the first equation. It can be deduced from Lemma \ref{AL1} and \ref{AL2} that:
\begin{align*}
&\mathbb{E}_{U^{2}}[s(\bm{x},\bm{x}')\ell(g(\bm{x}),+1)]-\frac{\pi_{-}}{\pi_{+}}\mathbb{E}_{U^{2}}[(1-s(\bm{x},\bm{x}'))\ell(g(\bm{x}),+1)]\\
&=\pi_{+}^{2}\mathbb{E}_{+}[\ell(g(\bm{x}),+1)]+\pi_{-}^{2}\mathbb{E}_{-}[\ell(g(\bm{x}),+1)]-\pi_{-}^{2}\mbox{\rm\large(}\mathbb{E}_{+}[\ell(g(\bm{x}),+1)]+\mathbb{E}_{-}[\ell(g(\bm{x}),+1)]\mbox{\rm\large)}\\
&=(\pi_{+}^{2}-\pi_{-}^{2})\mathbb{E}_{+}[\ell(g(\bm{x},+1))]\\
&=(\pi_{+}-\pi_{-})\mathbb{E}_{+}[\ell(g(\bm{x},+1))]
\end{align*}
Dividing each side by $\pi_{+}-\pi_{-}$, we can get an equivalent expression of $R_{+}(g)$ and $\hat{R}_{+}(g)$ is its unbiased estimator, which we can conclude the proof of the first equation.

The proof of the second equation is omitted since it can be proved in a completely symmetric way. As shown in \cite{SU}, though any convex combination of the loss terms of $\bm{x}$ and $\bm{x}'$ can be the unbiased estimator, the formulation above can achieve minimal variance among all the potential candidates, which can be helpful for better generalization.
\end{proof}
\section{Proof of Theorem \ref{bound}}
\label{A3}
For convenience, we make the following notations: 
\begin{align*}
\mathcal{L}_{Sconf}(g,(\bm{x},\bm{x}'))\mathop{=}\limits^{\triangle}&\frac{(s(\bm{x},\bm{x}')-\pi_{-})(\ell(g(\bm{x}),+1)+\ell(g(\bm{x}'),+1))}{2(\pi_{+}-\pi_{-})}\\
&-\frac{(s(\bm{x},\bm{x}')-\pi_{+})(\ell(g(\bm{x}),-1)+\ell(g(\bm{x}'),-1))}{2(\pi_{+}-\pi_{-})}
\end{align*}
Denote the Sconf data pairs of size $n$ with $S_{n}\mathop{\sim}\limits^{i.i.d.}p(\bm{x},\bm{x}')$. We first introduce the Rademacher complexity and give the following technical lemma:
\begin{definition}{\rm(Rademacher complexity \cite{rade}))}. Let $\bm{x}_{1},\cdots,\bm{x}_{n}$ be \textit{i.i.d.} random variables drawn from a probability distribution $\mathcal{D}$, $\mathcal{G}=\{g:~\mathcal{X}\rightarrow\mathbb{R}\}$ be a class of measurable functions. Then the Rademacher complexity of $\mathcal{G}$ is defined as:
\begin{align}
\mathfrak{R}_{n}(\mathcal{G})=\mathbb{E}_{\bm{x}_{1},\cdots,\bm{x}_{n}}\mathbb{E}_{\bm{\sigma}}\left[\sup_{g\in\mathcal{G}}\frac{1}{n}\sum_{i=1}^{n}\sigma_{i}g(\bm{x}_{i})\right].
\end{align}
where $\bm{\sigma}=(\sigma_{1}, . . . , \sigma_{n})$ are Rademacher variables taking from $\{-1, +1\}$ uniformly.
\end{definition}
\begin{lemma}
\label{rade}
\begin{align*}
\bar{\mathfrak{R}}_{n}(\mathcal{L}_{Sconf})\leq\frac{L_{\ell}}{|\pi_{+}-\pi_{-}|}\mathfrak{R}_{n}(\mathcal{G})
\end{align*}
where $\bar{\mathfrak{R}}_{n}(\mathcal{L}_{Sconf})$ is the Rademacher complexity of $\mathcal{L}_{Sconf}$ over Sconf data pairs of size n drawn from $U^{2}$.
\end{lemma}
\begin{proof}
Due to the sub-additivity of supremum, symmetry between $\bm{x}$ and $\bm{x}'$ and the property of Rademacher variable:
\begin{align*}
\bar{\mathfrak{R}}_{n}(\mathcal{L}_{Sconf})&=\mathbb{E}_{\mathcal{S}_{n}}\mathbb{E}_{\bm{\sigma}}\left[\sup_{g\in\mathcal{G}}\sum_{i=1}^{n}\sigma_{i}\frac{\mathcal{L}_{Sconf}(g,\bm{x}_{i},\bm{x}_{i}')}{n}\right]\\
&\leq \mathbb{E}_{\mathcal{S}_{n}}\mathbb{E}_{\bm{\sigma}}\left[\sup_{g\in\mathcal{G}}\sum_{i=1}^{n}\sigma_{i}\frac{(s(\bm{x}_{i},\bm{x}'_{i})-\pi_{-})\ell(g(\bm{x}_{i}),+1)+(\pi_{+}-(s(\bm{x}_{i},\bm{x}_{i}'))\ell(g(\bm{x}_{i}),-1))}{n(\pi_{+}-\pi_{-})}\right]
\end{align*}
Suppose $\pi_{+}>\pi_{-}$ . We also have the following results:
\begin{align}
\label{tem}
&\left\|\nabla\left(\frac{(s(\bm{x},\bm{x}')-\pi_{-})\ell(g(\bm{x}),+1)+(\pi_{+}-s(\bm{x},\bm{x}'))\ell(g(\bm{x}),-1))}{(\pi_{+}-\pi_{-})}\right)\right\|_{2}\\
&\leq\left\|\nabla\left(\frac{(s(\bm{x},\bm{x}')-\pi_{-})\ell(g(\bm{x}),+1)}{(\pi_{+}-\pi_{-})}\right)\right\|_{2}+\left\|\nabla\left(\frac{(\pi_{+}-(s(\bm{x},\bm{x}'))\ell(g(\bm{x}),-1))}{(\pi_{+}-\pi_{-})}\right)\right\|_{2}\nonumber\\
&\leq\frac{|s(\bm{x},\bm{x}')-\pi_{-}|L_{\ell}}{(\pi_{+}-\pi_{-})}+\frac{|\pi_{+}-(s(\bm{x},\bm{x}')|L_{\ell}}{(\pi_{+}-\pi_{-})}\nonumber
\end{align}
We can further bound (\ref{tem}) under different conditions:
$$
\frac{|s(\bm{x},\bm{x}')-\pi_{-}|L_{\ell}}{(\pi_{+}-\pi_{-})}+\frac{|\pi_{+}-(s(\bm{x},\bm{x}')|L_{\ell}}{(\pi_{+}-\pi_{-})}\leq
\begin{cases}
~~~~L_{\ell},~~~~~~~s(\bm{x},\bm{x}')\in[\pi_{-},~\pi_{+}],\\
\frac{L_{\ell}}{|\pi_{+}-\pi_{-}|},~~~~s(\bm{x},\bm{x}')\not\in[\pi_{-},~\pi_{+}].
\end{cases}
$$
which shows that (\ref{tem}) is upper bounded by $\frac{L_{\ell}}{|\pi_{+}-\pi_{-}|}$
According to Talagrand's lemma\cite{tala} and the result above, we can further get the following inequality:
\begin{align*}
\bar{\mathfrak{R}}_{n}(\mathcal{L}_{Sconf})&\leq \frac{L_{\ell}}{|\pi_{+}-\pi_{-}|}\mathbb{E}_{\mathcal{S}_{n}}\mathbb{E}_{\bm{\sigma}}\left[\sup_{g\in\mathcal{G}}\sum_{i=1}^{n}\sigma_{i}g(\bm{x}_{i})\right]\\
&=\frac{L_{\ell}}{|\pi_{+}-\pi_{-}|}\mathbb{E}_{\mathcal{X}_{n}}\mathbb{E}_{\bm{\sigma}}\left[\sup_{g\in\mathcal{G}}\sum_{i=1}^{n}\sigma_{i}g(\bm{x}_{i})\right]\\
&=\frac{L_{\ell}}{|\pi_{+}-\pi_{-}|}\mathfrak{R}_{n}(\mathcal{G})
\end{align*}
\end{proof}
Then we can bound $\sup_{g\in\mathcal{G}}\left|\hat{R}(g)-R(g)\right|$ using McDiarmid's inequality:
\begin{lemma}
\label{UC}
The inequalities below hold with probability at least $1-\delta$:
\begin{align}
\sup_{g\in\mathcal{G}}\left|R(g)-\hat{R}(g)\right|\leq\frac{L_{\ell}}{|\pi_{+}-\pi_{-}|}\mathfrak{R}_{n}(\mathcal{G})+\frac{C_{\ell}}{|\pi_{+}-\pi_{-}|}\sqrt{\frac{\ln 2/\delta}{2n}}.
\end{align}
\end{lemma}
\begin{proof}
To begin with, we first bound the one-side supremum $\sup_{g\in\mathcal{G}}\left(R(g)-\hat{R}(g)\right)$. Denote $\Phi=\sup_{g\in\mathcal{G}}\left(R(g)-\hat{R}(g)\right)$ and $\bar{\Phi}=\sup_{g\in\mathcal{G}}\left(R(g)-\hat{\bar{R}}(g)\right)$, where $\hat{R}(g)$ and $\hat{\bar{R}}(g)$ are empirical risk over two samples differing by exactly one point: $\{(\bm{x}_{n},\bm{x}_{n}'),s_{n}\}$ and $\{(\bar{\bm{x}}_{n},\bar{\bm{x}}_{n}'),\bar{s}_{n}\}$. Then we have:
\begin{align*}
\bar{\Phi}-\Phi&\leq\sup_{g\in\mathcal{G}}\left(\hat{R}(g)-\hat{\bar{R}}(g)\right)\\
&\leq\sup_{g\in\mathcal{G}}\left(\frac{\mathcal{L}_{Sconf}(g,\bm{x}_{n},\bm{x}'_{n})-\mathcal{L}_{Sconf}(g,\bar{\bm{x}}_{n},\bar{\bm{x}}'_{n})}{n}\right)\\
&\leq\frac{C_{\ell}}{n|\pi_{+}-\pi_{-}|}
\end{align*}
and $\Phi-\bar{\Phi}$ has the same upper bound symmetrically. By applying McDiarmid's inequality, the inequality below holds with probability at least $1-\frac{\delta}{2}$:
\begin{align}
\sup_{g\in\mathcal{G}}\left(R(g)-\hat{R}(g)\right)\leq\mathbb{E}_{S_{n}}\left[\sup_{g\in\mathcal{G}}\left(R(g)-\hat{R}(g)\right)\right]+\frac{C_{\ell}}{|\pi_{+}-\pi_{-}|}\sqrt{\frac{\ln 2/\delta}{2n}}.
\end{align}
The following step is to bound $\mathbb{E}_{S_{n}}\left[\sup_{g\in\mathcal{G}}\left(R(g)-\hat{R}(g)\right)\right]$ with Rademacher complexity. It is a routine work to show by symmetrization\cite{foundations} and Lemma \ref{rade} that 
\begin{align*}
\mathbb{E}_{S_{n}}\left[\sup_{g\in\mathcal{G}}\left(R(g)-\hat{R}(g)\right)\right]&\leq
\bar{\mathfrak{R}}_{n}(\mathcal{L}_{Sconf})\\
&\leq\frac{L_{\ell}}{|\pi_{+}-\pi_{-}|}\bar{\mathfrak{R}}_{n}(\mathcal{G})
\end{align*}
The other direction $\sup_{g\in\mathcal{G}}\left(\hat{R}(g)-R(g)\right)$ is similar. Using the union bound, the following inequality holds with probability at least $1-\delta$:
\begin{align}
\sup_{g\in\mathcal{G}}\left|R(g)-\hat{R}(g)\right|\leq\frac{L_{\ell}}{|\pi_{+}-\pi_{-}|}\mathfrak{R}_{n}(\mathcal{G})+\frac{C_{\ell}}{|\pi_{+}-\pi_{-}|}\sqrt{\frac{\ln 2/\delta}{2n}}.
\end{align}
\end{proof}
Then we can prove Theorem \ref{bound}:
\begin{proof}
\begin{align*}
R(\hat{g})-R(g^{*})&=(R(\hat{g})-\hat{R}(\hat{g}))+(\hat{R}(\hat{g})-\hat{R}(g^{*}))+(\hat{R}(g^{*})-R(g^{*}))\\
&\leq (R(\hat{g})-\hat{R}(\hat{g}))+(\hat{R}(g^{*})-R(g^{*}))\\
&\leq |R(\hat{g})-\hat{R}(\hat{g})|+|\hat{R}(g^{*})-R(g^{*})|\\
&\leq 2\sup_{g\in\mathcal{G}}|R(g)-\hat{R}(g)|
\end{align*}
The first inequality holds due to the definition of ERM. We can conclude the proof by applying Lemma \ref{UC}.
\end{proof}
\section{Proof of Theorem \ref{CP}}
\label{A4}
\begin{proof}
We first prove the unbiasedness of proposed class-prior estimator:
\begin{align*}
\mathbb{E}_{\mathcal{S}_{n}}\left[\frac{\sum_{i=1}^{n}s(\bm{x}_{i},\bm{x}_{i}')}{n}\right]&=\sum_{i=1}^{n}\mathbb{E}_{U^{2}}\left[\frac{s(\bm{x},\bm{x}')}{n}\right]=\mathbb{E}_{U^{2}}[s(\bm{x},\bm{x}')]\\
&=\int\int p(y=+1|\bm{x})p(y'=+1|\bm{x}')p(\bm{x})p(\bm{x}')d\bm{x}d\bm{x}'\\
&~~~~+\int\int p(y=-1|\bm{x})p(y'=-1|\bm{x}')p(\bm{x})p(\bm{x}')d\bm{x}d\bm{x}'\\
&=\int\int p(\bm{x},y=+1)p(\bm{x}',y'=+1)d\bm{x}d\bm{x}'\\
&~~~~+\int\int p(\bm{x},y=-1)p(\bm{x}',y'=-1)d\bm{x}d\bm{x}'\\
&=p(y=+1)p(y'=+1)+p(y=-1)p(y'=-1)\\
&=\pi_{+}^{2}+\pi_{-}^{2}
\end{align*}
Note that for any different Sconf data pairs $(\bm{x},\bm{x}')$ and $(\bar{\bm{x}},\bar{\bm{x}}')$: $\frac{|s(\bm{x},\bm{x}')-s(\bar{\bm{x}},\bar{\bm{x}}')|}{n}\leq\frac{1}{n}$. Then we can simply prove the consistency of proposed estimator using McDiarmid's inequality, which can be formulated as the following theorem:
\begin{theorem}
For any $\delta>0$ and $\mathcal{S}_{n}\mathop{\sim}\limits^{\textit{i.i.d.}}U^{2n}$, the following inequality holds with probability at least $1-\delta$:
\begin{align}
\left|\sum_{i=1}^{n}\frac{s(\bm{x}_{i},\bm{x}_{i}')}{n}-(\pi_{+}^{2}+\pi_{-}^{2})\right|\leq \sqrt{\frac{\ln 2/\delta}{2n}}
\end{align}
\end{theorem}
\end{proof}
\section{Proof of Theorem \ref{noise}}
\begin{proof}
According to the definition of empirical minimizers $\bar{g}$, $\hat{g}$ and the proof of Theorem \ref{bound}: 
\begin{align*}
R(\bar{g})-R(g^{*})=&\left(R(\bar{g})-\hat{R}(\bar{g})\right)+\left(\hat{R}(\bar{g})-\bar{R}(\bar{g})\right)+\left(\bar{R}(\bar{g})-\bar{R}(\hat{g})\right)+\left(\bar{R}(\hat{g})-\hat{R}(\hat{g})\right)\\
&+\left(\hat{R}(\hat{g})-R(\hat{g})\right)+\left(R(\hat{g})-R(g^{*})\right)\\
&\leq 2\sup\limits_{g\in\mathcal{G}}|R(g)-\hat{R}(g)|+2\sup\limits_{g\in\mathcal{G}}|\bar{R}(g)-\hat{R}(g)|+\left(\bar{R}(\hat{g})-\hat{R}(\hat{g})\right)\\
&\leq 4\sup\limits_{g\in\mathcal{G}}|R(g)-\hat{R}(g)|+2\sup\limits_{g\in\mathcal{G}}|\left(\bar{R}(g)-\hat{R}(g)\right)|\\
&\leq 4\sup\limits_{g\in\mathcal{G}}|R(g)-\hat{R}(g)|+\frac{2\sum_{i=1}^{n}C_{\ell}\sigma_{n}}{n(\pi_{+}-\pi_{-})}
\end{align*}
According to Lemma \ref{UC}, the following inequality holds with probability at least $1-\delta$:
$$R(\bar{g})\!-\!R(g^{*})\leq\frac{4L_{\ell}\mathfrak{R}_{n}(\mathcal{G})}{|\pi_{+}-\pi_{-}|}\!+\!\frac{4C_{\ell}}{|\pi_{+}-\pi_{-}|}\sqrt{\frac{\ln 2/\delta}{2n}}\!+\!\frac{2C_{\ell}\sigma_{n}}{n|\pi_{+}-\pi_{-}|},$$
which concludes the proof.
\end{proof}

\section{Proof of Theorem \ref{CRE}}
\label{AT4}
Denote the Sconf data pairs of size n with $\mathcal{S}_{n}=\{(\bm{x},_{i},\bm{x}'_{i})\}_{i=1}^{n}$. We first make the following notation: $\mathfrak{D}^{n}_{-}(g)=\{\mathcal{S}_{n}|\hat{R}_{+}(g)<0\}\cup\{\mathcal{S}_{n}|\hat{R}_{-}(g)<0\}$, $\mathfrak{D}^{n}_{+}(g)=\{\mathcal{S}_{n}|\hat{R}_{+}(g)\geq 0\}\cap\{\mathcal{S}_{n}|\hat{R}_{-}(g)\geq 0\}$, $R_{+}(g)=\pi_{+}\mathbb{E}_{+}[\ell(g(\bm{x}),+1)]$, $R_{-}(g)=\pi_{-}\mathbb{E}_{-}[\ell(g(\bm{x}),-1)]$. Before proving Theorem \ref{CRE}, we begin with the proof of a technical lemma. 
\begin{lemma}Assume that there is $\alpha>0$ and $\beta>0$ such that $R_{+}(g)\geq\alpha$ and $R_{-}(g)\geq\beta$. By assumptions in Theorem \ref{bound}, the probability measure of $\mathfrak{D}_{-}(g)$ can be upper bounded by:
$$\mathbb{P}(\mathfrak{D}_{-}(g))\leq\exp\left(-\frac{(\pi_{+}-\pi_{-})^{2}n}{2C_{\ell}^{2}}\right)\Delta$$
where $\Delta=\exp(\alpha^{2})+\exp(\beta^{2})$.
\end{lemma}
\begin{proof}
According to the data generation process:
$$p(S_{n})=p(\bm{x}_{1})\cdots p(\bm{x}_{n})p(\bm{x}'_{1})\cdots p(\bm{x}'_{n}),$$
and the probability measure of $\mathfrak{D}_{-}(g)$ is defined as below:
$$\mathbb{P}(\mathfrak{D}_{-}(g))=\int_{\mathcal{S}_{n}\in\mathfrak{D}_{-}(g)}p(S_{n})dS_{n}=\int_{\mathcal{S}_{n}\in\mathfrak{D}_{-}(g)}p(S_{n})d\bm{x}_{1}\cdots d\bm{x}_{n}d\bm{x}'_{1}\cdots d\bm{x}'_{n},$$
where $\mathbb{P}$ is the probability. 

By assumptions in Theorem \ref{bound}, the change of $\hat{R}_{+}(g)$ and $\hat{R}_{-}(g)$ will be no more that $2C_{\ell}/n|\pi_{+}-\pi_{-}|$ if exactly one pair of Sconf data $(\bm{x}_{i},\bm{x}'_{i})\in\mathcal{S}_{n}$ is replaced. According to McDiarmid's inequality:
$$\mathbb{P}(R_{+}(g)-\hat{R}_{+}(g)\geq \alpha)\leq \exp\left(-\frac{\alpha^{2}(\pi_{+}-\pi_{-})^{2}n}{2C_{\ell}^{2}}\right)$$
and
$$\mathbb{P}(R_{-}(g)-\hat{R}_{-}(g)\geq \beta)\leq \exp\left(-\frac{\beta^{2}(\pi_{+}-\pi_{-})^{2}n}{2C_{\ell}^{2}}\right)$$
Then we can bound $\mathbb{P}(\mathfrak{D}_{-}(g))$ in this manner:
\begin{align*}
\mathbb{P}(\mathfrak{D}_{-}(g))&\leq\mathbb{P}(\hat{R}_{+}(g)\leq 0)+\mathbb{P}(\hat{R}_{-}(g)\leq 0)\\
&\leq\mathbb{P}(\hat{R}_{+}(g)\leq R_{+}(g)-\alpha)+\mathbb{P}(\hat{R}_{-}(g)\leq R_{-}(g)-\beta)\\
&=\mathbb{P}(R_{+}(g)-\hat{R}_{+}(g)\geq \alpha)+\mathbb{P}(R_{-}(g)-\hat{R}_{-}(g)\geq \beta)\\
&\leq\exp\left(-\frac{\alpha^{2}(\pi_{+}-\pi_{-})^{2}n}{2C_{\ell}^{2}}\right)+\exp\left(-\frac{\beta^{2}(\pi_{+}-\pi_{-})^{2}n}{2C_{\ell}^{2}}\right)\\
&=\exp\left(-\frac{(\pi_{+}-\pi_{-})^{2}n}{2C_{\ell}^{2}}\right)\Delta
\end{align*}
The first inequality holds due to union bound and the second one is deduced according to the assumptions.
\end{proof}
Then we prove the Theorem \ref{CRE}:
\begin{proof}
According to the definition of consistent correction function:
\begin{align*}
&\mathbb{E}[\widetilde{R}(g)]-R(g)=\mathbb{E}[\widetilde{R}(g)-\hat{R}(g)]\\
&=\int_{S_{n}\in\mathfrak{D}_{+}(g)}(\widetilde{R}(g)-\hat{R}(g))p(S_{n})dS_{n}+\int_{S_{n}\in\mathfrak{D}_{-}(g)}(\widetilde{R}(g)-\hat{R}(g))p(S_{n})dS_{n}\\
&=\int_{S_{n}\in\mathfrak{D}_{-}(g)}(\widetilde{R}(g)-\hat{R}(g))p(S_{n})dS_{n}
\end{align*}
According to the definition of $\widetilde{R}(g)$, we know that it can upper bound $\hat{R}(g)$: $\widetilde{R}(g)\geq\hat{R}(h)$. Then we can get the \textit{l.h.s.} inequality:
$$\mathbb{E}[\widetilde{R}(g)-\hat{R}(g)]\geq 0$$
Note that the consistent correction function is Lipschitz continuous with Lipschitz constant $L_{f}=\max\{1,k\}$ and $f(0)=0$. Then we upper bound $\mathbb{E}[\widetilde{R}(g)]-R(g)$ based on the assumptions in Theorem \ref{bound} and the fact that $\left|\hat{R}_{+}(g)\right|$ and $\left|\hat{R}_{-}(g)\right|$ can be bounded by $2C_{\ell}/|\pi_{+}-\pi_{-}|$:
\begin{align*}
\mathbb{E}[\widetilde{R}(g)]-R(g)&=\int_{S_{n}\in\mathfrak{D}_{-}(g)}(\widetilde{R}(g)-\hat{R}(g))p(S_{n})dS_{n}\\
&\leq\sup_{S_{n}\in\mathfrak{D}_{-}(g)}\left((\widetilde{R}(g)-\hat{R}(g)\right)\int_{S_{n}\in\mathfrak{D}_{-}(g)}p(S_{n})dS_{n}\\
&=\sup_{S_{n}\in\mathfrak{D}_{-}(g)}\left((\widetilde{R}(g)-\hat{R}(g)\right)\mathbb{P}(\mathfrak{D}_{-}(g))\\
&=\sup_{S_{n}\in\mathfrak{D}_{-}(g)}\left(f\left(\hat{R}_{+}(g)\right)+f\left(\hat{R}_{-}(g)\right)-\hat{R}_{+}(g)-\hat{R}_{-}(g)\right)\mathbb{P}(\mathfrak{D}_{-}(g))\\
&\leq\sup_{S_{n}\in\mathfrak{D}_{-}(g)}\left(L_{f}\left|\hat{R}_{+}(g)\right|+L_{f}\left|\hat{R}_{-}(g)\right|+\left|\hat{R}_{+}(g)\right|+\left|\hat{R}_{-}(g)\right|\right)\mathbb{P}(\mathfrak{D}_{-}(g))\\
&\leq\sup_{S_{n}\in\mathfrak{D}_{-}(g)}\left(\frac{(L_{f}+1)C_{\ell}}{|\pi_{+}-\pi_{-}|}\right)\mathbb{P}(\mathfrak{D}_{-}(g))\\
&=\frac{(L_{f}+1)C_{\ell}}{|\pi_{+}-\pi_{-}|}\exp\left(-\frac{(\pi_{+}-\pi_{-})^{2}n}{2C_{\ell}^{2}}\right)\Delta
\end{align*}
Then we give the high-probability bound of consistent risk estimator $\widetilde{R}(g)$ by bounding $\left|\widetilde{R}(g)-R(g)\right|$. We first give the following inequality according to the discussions above:
\begin{align}
\left|\widetilde{R}(g)-R(g)\right|&\leq\left|\widetilde{R}(g)-\mathbb{E}[\widetilde{R}(g)]\right|+\left|\mathbb{E}[\widetilde{R}(g)]-R(g)\right|\nonumber\\
\label{Tech}
&\leq\left|\widetilde{R}(g)-\mathbb{E}[\widetilde{R}(g)]\right|+\frac{(L_{f}+1)C_{\ell}}{|\pi_{+}-\pi_{-}|}\exp\left(-\frac{(\pi_{+}-\pi_{-})^{2}n}{2C_{\ell}^{2}}\right)\Delta
\end{align} 
Then we can focus on bounding $\left|\widetilde{R}(g)-\mathbb{E}[\widetilde{R}(g)]\right|$. According to the definition of $\widetilde{R}(g)$ and the Lipschitzness of $f(\cdot)$, the change of $\widetilde{R}(g)$ will be no more than $L_{\ell}C_{\ell}/n|\pi_{+}-\pi_{-}|$. Then we can simply bound $\left|\widetilde{R}(g)-\mathbb{E}[\widetilde{R}(g)]\right|$ using McDiarmid's inequality. With probability at least $1-\delta$, the following inequality holds:
$$\left|\widetilde{R}(g)-\mathbb{E}[\widetilde{R}(g)]\right|\leq \frac{L_{\ell}C_{\ell}}{|\pi_{+}-\pi_{-}|}\sqrt{\frac{\ln2/\delta}{2n}}$$
We can conclude the proof by combining the inequality above and (\ref{Tech}).    
\end{proof}
\section{Proof of Theorem \ref{CON}}
\label{AT5}
Based on Theorem \ref{CRE} and the proof of Theorem \ref{bound}, we prove Theorem \ref{CON}:
\begin{proof}
We first give the following inequalities:
\begin{align*}
R(\tilde{g})-R(g^{*})&=\left(R(\tilde{g})-\widetilde{R}(\tilde{g})\right)+\left(\widetilde{R}(\tilde{g})-\widetilde{R}(\hat{g})\right)+\left(\widetilde{R}(\hat{g})-R(\hat{g})\right)+\left(R(\hat{g})-R(g^{*})\right)\\
&\leq\left|R(\tilde{g})-\widetilde{R}(\tilde{g})\right|+\left|
\widetilde{R}(\hat{g})-R(\hat{g})\right|+\left(R(\hat{g})-R(g^{*})\right)
\end{align*}
Then we can conclude the proof by combining the high-probability bound in Theorem \ref{CRE}, Theorem \ref{bound} and union bound. With probability at least $1-\delta$, the following inequality holds:
\begin{align*}
R(\tilde{g})-R(g^{*})&\leq\left|R(\tilde{g})-\widetilde{R}(\tilde{g})\right|+\widetilde{R}(\tilde{g})-\widetilde{R}(\hat{g})+\left|
\widetilde{R}(\hat{g})-R(\hat{g})\right|+\left(R(\hat{g})-R(g^{*})\right)\\
&\leq \frac{2L_{\ell}}{|\pi_{+}-\pi_{-}|}\mathfrak{R}_{n}(\mathcal{G})+\sqrt{\frac{\ln 6/\delta}{2n}}\left(\frac{2L_{\ell}C_{\ell}+2C_{\ell}}{|\pi_{+}-\pi_{-}|}\right)+\frac{2(L_{f}+1)C_{\ell}}{|\pi_{+}-\pi_{-}|}\exp\left(-\frac{(\pi_{+}-\pi_{-})^{2}n}{2C_{\ell}^{2}}\right)\Delta
\end{align*}
\end{proof}
\section{Symmetric Conclusions of Theorem \ref{TS} and Theorem \ref{coll} for Dissimilar Data Pairs}
Suppose the dissimilar data pairs $\{(\bm{x}_{i},\bm{x}'_{i})\}_{i=1}^{n}$ are drawn from the distribution with density $p_{D}(\bm{x},\bm{x}')=p(\bm{x},\bm{x}'|y\not=y')$. We give an unbiased risk estimator of classification risk with only dissimilar data pairs and their similarity confidence:
\begin{theorem}
With dissimilar data pairs and their similarity confidence, assuming that $s(\bm{x},\bm{x}')<1$ for all the pair $(\bm{x},\bm{x}')$, we can get the unbiased estimator of classification risk (\ref{OR}), i.e.,  ${\textstyle\mathbb{E}_{p(\bm{x},\bm{x}'|y\not=y')}[\hat{R}_{\mathrm{D}}(g)]=R(g)}$, where
\begin{align}
\hat{R}_{\mathrm{D}}(g)=\textstyle2\pi_{+}\pi_{-}\sum\nolimits_{i=1}^{n}\frac{(s_{i}-\pi_{-})(\ell(g(\bm{x}_{i}),+1)+\ell(g(\bm{x}'_{i}),+1))}{2n(\pi_{+}-\pi_{-})(1-s_{i})}+2\pi_{+}\pi_{-}\sum\nolimits_{i=1}^{n}\frac{(\pi_{+}-s_{i})(\ell(g(\bm{x}_{i}),-1)+\ell(g(\bm{x}'_{i}),-1))}{2n(\pi_{+}-\pi_{-})(1-s_{i}}).
\end{align}
\end{theorem}
\begin{proof}
First we show the equivalent expression of $p(\bm{x},\bm{x'}|y\not=y')$. According to the independence assumption $(\bm{x},y)\perp(\bm{x}',y')$, we can immediately get the independence between $\bm{x},~\bm{x}'$ and $y,~y'$. Then the following equations hold:
\begin{align*}
p_{D}(\bm{x},\bm{x}')&=p(\bm{x},\bm{x}'|y\not=y')=\frac{p(\bm{x},\bm{x}',y\not=y')}{p(y=y')}\\
&=\frac{p(\bm{x},y=+1,\bm{x}',y'=-1)+p(\bm{x},y=-1,\bm{x}',y'=+1)}{p(y=+1)p(y'=-1)+p(y=-1)p(y'=+1)}\\
&=\frac{p(\bm{x},y=+1)p(\bm{x}',y'=-1)+p(\bm{x},y=-1)p(\bm{x}',y'=+1)}{2\pi_{+}\pi_{-}}\\
&=\frac{\pi_{+}\pi_{-}p_{+}(\bm{x})p_{-}(\bm{x}')+\pi_{+}\pi_{-}p_{-}(\bm{x})p_{+}(\bm{x}')}{2\pi_{+}\pi_{-}}\\
&=\frac{p_{+}(\bm{x})p_{-}(\bm{x}')+p_{-}(\bm{x})p_{+}(\bm{x}')}{2}
\end{align*}

Denote $2\pi_{+}\pi_{-}$ with $\pi_{\mathrm{D}}$. Then we can prove the theorem above. 
\begin{align*}
&\mathbb{E}_{p_{D}(\bm{x},\bm{x}')}\left[\frac{\pi_{D}(s(\bm{x},\bm{x}')-\pi_{-})(\ell(g(\bm{x}),+1)+\ell(g(\bm{x}'),+1))}{2(\pi_{+}-\pi_{-})(1-s(\bm{x},\bm{x}'))}\right]\\
&=\int\frac{\pi_{D}(s(\bm{x},\bm{x}')-\pi_{-})(\ell(g(\bm{x}),+1)+\ell(g(\bm{x}'),+1))}{2(\pi_{+}-\pi_{-})(1-s(\bm{x},\bm{x}'))}*\frac{p_{+}(\bm{x})p_{-}(\bm{x}')+p_{-}(\bm{x})p_{+}(\bm{x}')}{2}d\bm{x}d\bm{x}'\\
&=\int\frac{(\pi_{+}^{2}p_{+}(\bm{x})p_{+}(\bm{x}')+\pi_{-}^{2}p_{-}(\bm{x})p_{-}(\bm{x}')-\pi_{-}p(\bm{x})p(\bm{x}'))(\ell(g(\bm{x}),+1)+\ell(g(\bm{x}'),+1))}{4(\pi_{+}-\pi_{-})}d\bm{x}d\bm{x}'\\
&=\int\frac{(\pi_{+}^{2}p_{+}(\bm{x})+\pi_{-}^{2}p_{-}(\bm{x})-\pi_{-}p(\bm{x}))\ell(g(\bm{x}),+1)}{2(\pi_{+}-\pi_{-})}d\bm{x}+\int\frac{(\pi_{+}^{2}p_{+}(\bm{x}')+\pi_{-}^{2}p_{-}(\bm{x}')-\pi_{-}p(\bm{x}'))\ell(g(\bm{x}'),+1)}{2(\pi_{+}-\pi_{-})}d\bm{x}'\\
&=\int\frac{\pi_{+}(\pi_{+}-\pi_{-})p_{+}(\bm{x})\ell(g(\bm{x}),+1)}{2(\pi_{+}-\pi_{-})}d\bm{x}+\int\frac{\pi_{+}(\pi_{+}-\pi_{-})p_{+}(\bm{x})\ell(g(\bm{x}'),+1)}{2(\pi_{+}-\pi_{-})}d\bm{x}'\\
&=\int\frac{\pi_{+}p_{+}(\bm{x})\ell(g(\bm{x}),+1)}{2}d\bm{x}+\int\frac{\pi_{+}p_{+}(\bm{x})\ell(g(\bm{x}'),+1)}{2}d\bm{x}'\\
&=\frac{\pi_{+}\mathbb{E}_{+}[\ell(g(\bm{x}),+1)]}{2}+\frac{\pi_{+}\mathbb{E}_{+}[\ell(g(\bm{x}),+1)]}{2}\\
&=\pi_{+}\mathbb{E}_{+}[\ell(g(\bm{x}),+1)]
\end{align*}

Symmetrically, we have:
\begin{align*}
&\mathbb{E}_{p_{D}(\bm{x},\bm{x}')}\left[\frac{\pi_{D}(\pi_{+}-s(\bm{x},\bm{x}'))(\ell(g(\bm{x}),-1)+\ell(g(\bm{x}'),-1))}{2(\pi_{+}-\pi_{-})(1-s(\bm{x},\bm{x}'))}\right]\\
&=\int\frac{\pi_{S}(\pi_{+}-s(\bm{x},\bm{x}'))(\ell(g(\bm{x}),-1)+\ell(g(\bm{x}'),-1))}{2(\pi_{+}-\pi_{-})(1-s(\bm{x},\bm{x}'))}*\frac{p_{+}(\bm{x})p_{-}(\bm{x}')+p_{-}(\bm{x})p_{+}(\bm{x}')}{2}d\bm{x}d\bm{x}'\\
&=\int\frac{(\pi_{+}p(\bm{x})p(\bm{x}')-\pi_{+}^{2}p_{+}(\bm{x})p_{+}(\bm{x}')-\pi_{-}^{2}p_{-}(\bm{x})p_{-}(\bm{x}'))(\ell(g(\bm{x}),-1)+\ell(g(\bm{x}'),-1))}{2(\pi_{+}-\pi_{-})}d\bm{x}d\bm{x}'\\
&=\int\frac{(\pi_{+}p(\bm{x})-\pi_{+}^{2}p_{+}(\bm{x})-\pi_{-}^{2}p_{-}(\bm{x}))\ell(g(\bm{x}),-1)}{2(\pi_{+}-\pi_{-})}d\bm{x}+\int\frac{(\pi_{+}p(\bm{x}')-\pi_{+}^{2}p_{+}(\bm{x}')-\pi_{-}^{2}p_{-}(\bm{x}'))\ell(g(\bm{x}'),-1)}{2(\pi_{+}-\pi_{-})}d\bm{x}'\\
&=\int\frac{\pi_{-}(\pi_{+}-\pi_{-})p_{-}(\bm{x})\ell(g(\bm{x}),-1)}{2(\pi_{+}-\pi_{-})}d\bm{x}+\int\frac{\pi_{-}(\pi_{+}-\pi_{-})p_{-}(\bm{x})\ell(g(\bm{x}'),-1)}{2(\pi_{+}-\pi_{-})}d\bm{x}'\\
&=\int\frac{\pi_{-}p_{-}(\bm{x})\ell(g(\bm{x}),-1)}{2}d\bm{x}+\int\frac{\pi_{-}p_{-}(\bm{x})\ell(g(\bm{x}'),-1)}{2}d\bm{x}'\\
&=\frac{\pi_{-}\mathbb{E}_{-}[\ell(g(\bm{x}),-1)]}{2}+\frac{\pi_{-}\mathbb{E}_{-}[\ell(g(\bm{x}),-1)]}{2}\\
&=\pi_{-}\mathbb{E}_{-}[\ell(g(\bm{x}),-1)]
\end{align*}
Then we have:
\begin{align}
R_{\mathrm{D}}(g)=&\mathbb{E}_{p_{D}(\bm{x},\bm{x}')}\left[\frac{\pi_{D}(s(\bm{x},\bm{x}')-\pi_{-})(\ell(g(\bm{x}),+1)+\ell(g(\bm{x}'),+1))}{2(\pi_{+}-\pi_{-})(1-s(\bm{x},\bm{x}'))}\right]
\nonumber\\
&+\mathbb{E}_{p_{D}(\bm{x},\bm{x}')}\left[\frac{\pi_{D}(\pi_{+}-s(\bm{x},\bm{x}'))(\ell(g(\bm{x}),-1)+\ell(g(\bm{x}'),-1))}{2(\pi_{+}-\pi_{-})(1-s(\bm{x},\bm{x}'))}\right].
\end{align}
and we can give the unbiased estimator of classification risk according to the risk expression above:
\begin{align}
\hat{R}_{\mathrm{D}}(g)=\textstyle2\pi_{+}\pi_{-}\sum\nolimits_{i=1}^{n}\frac{(s_{i}-\pi_{-})(\ell(g(\bm{x}_{i}),+1)+\ell(g(\bm{x}'_{i}),+1))}{2n(\pi_{+}-\pi_{-})(1-s_{i})}+2\pi_{+}\pi_{-}\sum\nolimits_{i=1}^{n}\frac{(\pi_{+}-s_{i})(\ell(g(\bm{x}_{i}),-1)+\ell(g(\bm{x}'_{i}),-1))}{2n(\pi_{+}-\pi_{-})(1-s_{i})}.
\end{align}
which concludes the proof.

\end{proof}
Denote the empirical risk minimizer of $\hat{R}_{\mathrm{D}}(g)$ with $\hat{g}_{\mathrm{D}}$. We theoretically show that learning with only dissimilar data pairs can result in collapsed solution:
\begin{theorem}
Suppose $\pi_{+}>\pi_{-}$ and  0-1 loss is used. For similar data pairs, we assume that $s_{i}\leq\pi_{-}$ for $i=1,\cdots,n$. Then $\hat{g}_{\mathrm{D}}$ is a collapsed solution that classifies all the examples as negative. 
\end{theorem}
\begin{proof}
We aim to solve the following optimization problem when conducting ERM algorithm according to Theorem \ref{TS}:
\begin{align}
\label{optt}
    \min\limits_{g\in\mathcal{G}}\pi_{D}\sum\nolimits_{i=1}^{n}\left(\frac{(s_{i}-\pi_{-})(\ell(g(\bm{x}_{i}),+1)+\ell(g(\bm{x}'_{i}),+1))}{2n(\pi_{+}-\pi_{-})(1-s_{i})}+\frac{(\pi_{+}-s_{i})(\ell(g(\bm{x}_{i}),-1)+\ell(g(\bm{x}'_{i}),-1))}{2n(\pi_{+}-\pi_{-})(1-s_{i})}\right).
\end{align}
Notice that since $\pi_{+}>\pi_{-}$ and $s_{i}\leq\pi_{-}$ for all $i\in[n]$, we have the following 
$$
\begin{cases}
\frac{s_{i}-\pi_{-}}{2n(\pi_{+}-\pi_{-})(1-s_{i})}\leq 0,~i=1\cdots,n \\
\frac{\pi_{+}-s_{i}}{2n(\pi_{+}-\pi_{-})(1-s_{i})}\geq 0,~i=1\cdots,n
\end{cases}
$$
Since 0-1 loss is used, we have the conclusion that $\ell(g(\bm{x}),y)\in[0,1]$ for any $g$, $\bm{x}$, and $y$. According to the discussion above, by setting all the $\ell(\cdot,-1)$ to 0 and $\ell(\cdot,+1)$ to 1, we can get the lower bound of (\ref{optt}):
$$(\ref{optt})\geq \sum_{i=1}^{n}\frac{\pi_{D}(s_{i}-\pi_{+})}{n(\pi_{+}-\pi_{-})(1-s_{i})}. $$
It is obvious that such setting can be realized if we let $g(\bm{x})<0$ for all the $\bm{x}$, which means that $g$ classifies all the examples as negative.
\end{proof}
\section{Additional Information of Experiments}
\subsection{Detailed Setup of Figure \ref{DD}}
We generated 500 positive data and 300 negative data according to the 2-dimensional Gaussian distributions with different means and covariance for $p_{+}(\bm{x})$ and $p_{-}(\bm{x})$. The parameters are listed below:
\begin{align*}
\bm{\mu}_{+}=[-4,~0]^{\top},~\bm{\mu}_{-}=[2,2]^{\top},~ \bm{\Sigma}_{+}=\begin{bmatrix}2&0\\0&2\end{bmatrix},~ \bm{\Sigma}_{-}=\begin{bmatrix}3&0\\0&3\end{bmatrix}.
\end{align*}
Adam was chosen as the optimizer with default momentum parameters ($\beta_{1}=0.9,~\beta_{2}=0.999$) and the learning rate, epoch, weight decay, and batch size were fixed to be 1e-1, 30, 1e-3, and 128, respectively.
\subsection{Detailed Setup of Synthetic Experiments}
In the synthetic experiments in Section \ref{SE}, we generate 4 synthetic datasets to show the validity of our methods. The detailed parameters for generating different synthetic datasets are listed below. $\bm{\mu}_{+}$ and $\bm{\mu}_{-}$ are the means for two Gaussian distributions and $\bm{\Sigma}_{+}$ and $\bm{\Sigma}_{-}$ are the covariance for two Gaussian distributions:
\begin{itemize}
    \item Setup A: $\bm{\mu}_{+}=[0,~0]^{\top}$, $\bm{\mu}_{-}=[-2,~5]^{\top}$, $\bm{\Sigma}_{+}=\begin{bmatrix}7&-6\\-6&7\end{bmatrix}$, $\bm{\Sigma}_{-}=\begin{bmatrix}2&0\\0&2\end{bmatrix}$. 
    \item Setup B: $\bm{\mu}_{+}=[0,~0]^{\top}$,  $\bm{\mu}_{-}=[4,~0]^{\top}$, ~~~~$\bm{\Sigma}_{+}=\begin{bmatrix}3&0\\0&3\end{bmatrix}$, $\bm{\Sigma}_{-}=\begin{bmatrix}2&0\\0&2\end{bmatrix}$.
    \item Setup C: $\bm{\mu}_{+}=[0,~0]^{T}$, $\bm{\mu}_{-}=[3,~-3]^{T}$, $\bm{\Sigma}_{+}=\begin{bmatrix}2&0\\0&2\end{bmatrix}$, $\bm{\Sigma}_{-}=\begin{bmatrix}4&-3\\-3&4\end{bmatrix}$.
    \item Setup D: $\bm{\mu}_{+}=[0,~0]^{T}$, $\bm{\mu}_{-}=[4,~4]^{T}$, ~~~~$\bm{\Sigma}_{+}=\begin{bmatrix}2&0\\0&2\end{bmatrix}$, $\bm{\Sigma}_{-}=\begin{bmatrix}6&-5\\-5&6\end{bmatrix}$.
\end{itemize}
\subsection{Detailed Setup of Benchmark Experiments}
In Section \ref{bench}, we use 8 widely-used large-scale benchmark datasets. The detailed statistics of the datasets and the corresponding models are listed in Table \ref{Ta}:
\begin{table*}[t]
\caption{Detailed Statistics of benchmark datasets and models}
\label{Ta}
\centering
\begin{tabular}{c|ccccc|c}
\toprule
Datasets&\# Train&\# Validation&\# Test&$\pi_{+}$&Dim&Model $g(\bm{x})$\\
\midrule
MNIST&54000&6000&10000&0.3&784&3-layer MLP with ReLU ($d$-500-500-1)\\
Kuzushiji-MNIST&54000&6000&10000&0.7&784&3-layer MLP with ReLU ($d$-500-500-1)\\
Fashion-MNIST&54000&6000&10000&0.4&784&3-layer MLP with ReLU  ($d$-500-500-1)\\
EMNIST-Digits&216000&24000&40000&0.6&784&3-layer MLP  with ReLU ($d$-500-500-1)\\
EMNIST-Letters&112320&12480&20800&0.6153&784&3-layer MLP  with ReLU ($d$-500-500-1)\\
EMNIST-Balanced&101520&11280&18800&0.5744&784&3-layer MLP  with ReLU ($d$-500-500-1)\\
CIFAR-10&54000&6000&10000&0.6&3072&ResNet-34\\
SVHN&65931&7326&26032&0.7085&3072&ResNet-18\\
\bottomrule
\end{tabular}
\end{table*}
We report the sources of these datasets and the way we corrupt these datasets into binary datasets.
\begin{itemize}
    \item MNIST \cite{Mnist}. It is a grayscale dataset of handwritten digits from 0 to 9, where the size of the images is 28*28. Source: \url{http://yann.lecun.com/exdb/mnist/}.
    
    The digits $0\sim 2$ are used as the positive class and the rest digits are used as the negative class. 
    \item Kuzushiji-MNIST \cite{Kmnist}. It is a 10-class dataset of cursive Japanese characters ('Kuzushiji'). Source: \url{https://github.com/rois-codh/kmnist}.

    The positive class includes 'O', 'Ki', 'Su', 'Tsu', 'Na', 'Ha', and 'Ma'. The negative class includes 'Ya', 'Re', and 'Wo'.
    
    \item Fashion-MNIST \cite{Fmnist}. It is a 10-class dataset of fashion items. Each instance is a 28*28 grayscale image. Source: \url{https://github.com/zalandoresearch/fashion-mnist}.
    
    'T-short', 'Pullover', 'Dress', and 'Shirt' make up the positive class and the negative class is made up of 'Trouser', 'Coat', 'Sandal', 'Sneaker', 'Bag', and 'Ankel boot' . 
    \item  EMNIST \cite{Emnist}. A dataset that contain both letters and digits. Source: \url{https://www.westernsydney.edu.au/icns/reproducible\_research/publication\_support\_materials/emnist}.
    
    The splits 'Digits', 'Letters', and 'Balanced' are used and the details of each split are listed below:
    \begin{itemize}
        \item For 'Digits', $0\sim 5$ are used as the positive class and $6\sim 9$ are used as the negative class;
        \item For 'Letters', 'a'$\sim$'p' are used as the positive class and 'q'$\sim$'z' are used as the negative class;
        \item For 'Balanced', instances with class labels in $[0,~26]$ are used as the rest of the instances are used as the negative class.
    \end{itemize}
    
    \item CIFAR-10 \cite{CIFAR-10}. It is a 10-class dataset for 10 different objects and each instance is a 32*32*3 colored image in RGB format. Source: \url{https://www.cs.toronto.edu/~kriz/cifar.html}.
    
    'Bird', 'Cat', 'Dog', 'Deer', 'Frog', and 'Horse' form the positive class. The negative class is formed by 'Airplane', 'Automobile', 'Ship', and 'Truck'.
    
    \item SVHN \cite{SVHN}, a real-world image dataset of digits from 0 to 9. Each instance is a 32*32*3 colored image in RGB format. Source: \url{http://ufldl.stanford.edu/housenumbers/}.
    
    The positive class is composed of digits $0\sim 5$ and the negative class is composed of $6\sim 9$. 

\end{itemize}
The hyper-parameters for optimization algorithms are shown below:

Adam with default momentum was used for optimization in this paper. For generating similarity confidence, the epoch number, batch size, and learning rate are 10, 3000, and 1e-2, respectively. 

For Sconf-Unbiased, Sconf-ABS, Sconf-NN, and SD, the epoch number, batch size, and weight decay are 60, 3000, and 1e-3, respectively. The initial learning rate was set to 1e-3 and divided by 10 every 20 epochs. 

For Siamese and Contrastive, the epoch number, batch size, weight decay, and learning rate are 10, 3000, 1e-3, and 1e-3.
\end{appendix}
\end{document}